\documentclass[sigconf]{acmart}

\usepackage{booktabs} 
\usepackage{hyperref}       
\usepackage{url}            
\usepackage{booktabs}       
\usepackage{amsfonts}       
\usepackage{nicefrac}       
\usepackage{microtype}      
\usepackage{graphicx,amsmath,amsfonts,amssymb,bbm,mathptm,time,dsfont,wasysym,amsthm,subfigure,caption,epstopdf} 
\usepackage[ruled]{algorithm}
\usepackage{algorithmicx,algpseudocode}    
\usepackage{verbatim,graphicx,enumitem,bm}
\setlist{leftmargin=-10.5mm}
\setlist{nosep,after=\vspace{\baselineskip}}
\newcommand{\dm}{DM}
\newcommand{\ie}{\textit{i.e.}}
\newcommand{\eg}{\textit{e.g.}}
\newcommand{\st}{\textit{s.t.}}

\newcommand{\It}{I_t}
\newcommand{\bmu}{\bm{\mu}}
\newcommand{\hbmu}{\hat{\bm{\mu}}}

\newcommand{\bnu}{\bm{\nu}}
\newcommand{\A}{\mathcal{A}}
\newcommand{\B}{\mathcal{B}}
\renewcommand{\C}{\mathcal{C}}
\newcommand{\D}{\mathcal{D}}
\newcommand{\E}{\mathbb{E}}

\newcommand{\V}{\mathcal{V}}
\renewcommand{\a}{\bm{a}}
\renewcommand{\b}{\bm{b}}
\renewcommand{\c}{\bm{c}}

\newcommand{\f}{\bm{f}}
\renewcommand{\r}{\bm{r}}
\newcommand{\w}{\bm{w}}
\newcommand{\x}{\bm{x}}
\newcommand{\z}{\bm{z}}
\newcommand{\argmax}{\textnormal{argmax}}
\newcommand{\argmin}{\textnormal{argmin}}
\newcommand{\etc}{\textnormal{EC algorithm}}
\newcommand{\ttc}{\textnormal{TTC algorithm}}
\newcommand{\mttc}{\textnormal{MTTC algorithm}}
\newtheorem{assumption}{Assumption}
\newtheorem{remark}{Remark}

\allowdisplaybreaks
\begin{document}
\title[Learning to Route Efficiently with End-to-End Feedback]{Learning to Route Efficiently with End-to-End Feedback: The Value of Networked Structure}

\author{Ruihao Zhu}
\affiliation{%
  \institution{IDSS, MIT}
}
\email{rzhu@mit.edu}

\author{Eytan Modiano}
\affiliation{%
  \institution{LIDS, MIT}
}
\email{modiano@mit.edu}

\renewcommand{\shortauthors}{Zhu and Modiano}
\begin{abstract}
We introduce efficient algorithms which achieve nearly optimal regrets for the problem of stochastic online shortest path routing with end-to-end feedback. The setting is a natural application of the combinatorial stochastic bandits problem, a special case of the linear stochastic bandits problem. We show how the difficulties posed by the large scale action set can be overcome by the networked structure of the action set. Our approach presents a novel connection between bandit learning and shortest path algorithms. Our main contribution is an adaptive exploration algorithm with nearly optimal instance-dependent regret for any directed acyclic network. We then modify it so that nearly optimal worst case regret is achieved simultaneously. Driven by the carefully designed Top-Two Comparison (TTC) technique, the algorithms are efficiently implementable. We further conduct extensive numerical experiments to show that our proposed algorithms not only achieve superior regret performances compared to existing algorithms, but also reduce the runtime drastically.
\end{abstract}

%
%



\maketitle

\section{Introduction}
\label{intro}
We study the problem of shortest path routing over a network, where the link delays are not known in advance. When delays are known, it is possible to compute the shortest path in polynomial time via the celebrated Dijkstra's algorithm \cite{D59} or the Bellman-Ford algorithm \cite{B58}. However, link delays are often unknown, and evolve over time according to some unknown stochastic process. Moreover, there are many real-world scenarios in which only the end-to-end delays are observable. For example, overlay network is a communication network architecture that integrates controllable overlay nodes into an uncontrollable underlay network of legacy devices. It is generally difficult to ensure individual link delay feedback when routing in an overlay network as the underlay nodes are not necessarily cooperative.  Fig. \ref{fig:overlay_eg} shows a very simple overlay network, where the only overlay nodes are the source node (node 1) and destination node (node 6); while the nodes within the dotted circle are underlay nodes. Here, the Decision Maker (\dm) can choose to route the packets from one of the five paths available, namely $(1,2,3,6),(1,2,5,6),(1,2,3,5,6),(1,2,4,5,6),$ and $(1,4,5,6)$. If it picks path $(1,4,5,6),$ it can only get the realized delay of the whole path $(1,4,5,6),$ but not any of the realized delays of link $(1,4),(4,5),$ or $(5,6).$ These uncertainties and the network architectural constraints make the problem fall into the category of stochastic online shortest path routing with end-to-end feedback \cite{TZCPJ18}.
\begin{figure}[!ht]
	\centering
	\includegraphics[width=8.5cm,height=5.5cm]{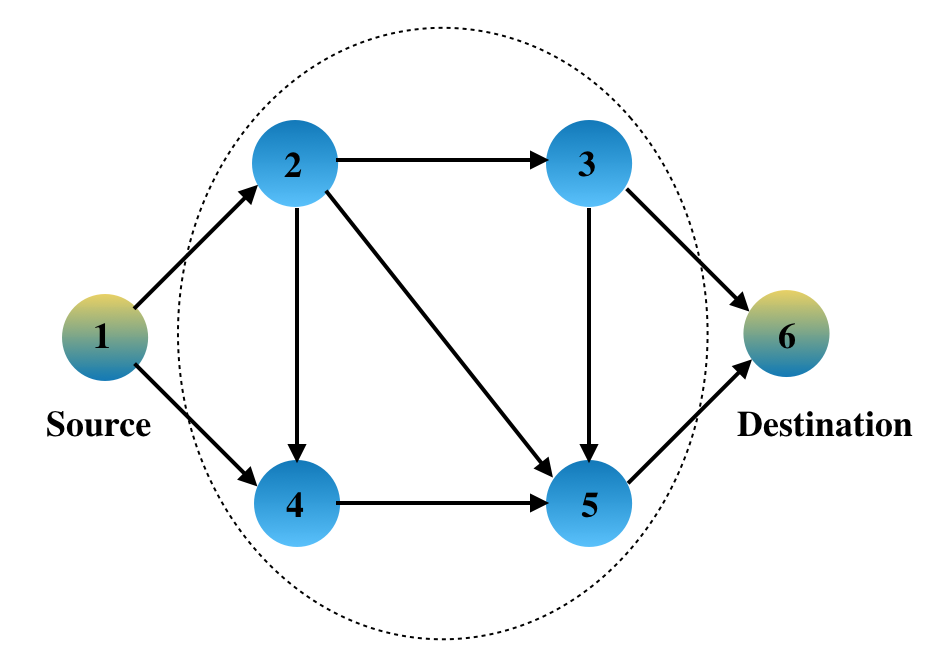}
	\vspace{-8mm}
	\caption{A toy example of an overlay network. Here, only the source and destination are overlay nodes. All other nodes belong to the underlay network.}
	\label{fig:overlay_eg}
\end{figure}

Stochastic online shortest path routing with end-to-end feedback is one of the most fundamental real-time decision-making problems. In its canonical form, a \dm~is presented a network with $d$ links, each link's delay is a random variable, following an unknown stochastic process with unknown fixed mean over $T$ rounds. In each round, a packet arrives to the \dm, and it chooses a path to route the packet from the source to the destination. The packet then incurs a delay, which is the sum of the delays realized on the associated links. Afterwards, the \dm~learns the end-to-end delay, \ie, the realized delay of the path, but the individual link's delay remains concealed. This is often called the \emph{bandit-feedback} setting \cite{TZCPJ18,KWAS15}. The \dm's goal is to design a routing policy that minimizes the cumulative expected delay. When the \dm~has full knowledge of the delay distributions, it would always choose to route the packets through the path with shortest expected delay. With that in mind, a reasonable performance metric for evaluating the policy is the \emph{expected regret}, defined to be the expected total delay of routing through the actual paths selected by the \dm~minus the expected total delay of routing through the path with shortest expected delay. In order to minimize the regret, the \dm~needs to learn the delay distributions on-the-fly. One viable approach to estimate the path delays is to inspect the end-to-end delays experienced by packets sent on different paths. This gives rise to an \emph{exploration-exploitation dilemma}. On one hand, the \dm~is not able to estimate the delay of an under-explored path; while on the other, the \dm~wants to send the the packet via the estimated shortest path to greedily minimize the cumulative delay incurred by the packets. 

The \emph{Upper Confidence Bound} (UCB) algorithm, following the \emph{Optimism-in-the-Face of Uncertainty} (OFU) principle, is one of the most prevalent strategies to deal with the exploration-exploitation dilemma. In the ordinary stochastic multi-armed bandits (MAB) settings, the UCB algorithm proposes a very intuitive policy framework, that \dm~should select actions by maximizing over rewards estimated from previous data but only after biasing each estimate according to its uncertainty. Simply put, one should choose the action that maximizes the ``mean plus confidence interval." Treating the inverse of delay as reward, a naive application of UCB algorithm to stochastic online shortest path routing can result in regret bounds and computation time that scale linearly with the number of paths. For small scale overlay networks, this achieves low regret efficiently. However, networks often have exponentially many paths, and direct implementation of the UCB algorithm is neither computationally efficient nor regret optimal. In the \emph{combinatorial semi-bandits} setting, the realized delay of each individual link on the chosen path is revealed. The authors of \cite{GKJ12} take the advantage of the individual feedbacks, and propose a solution for the problem by computing the UCB of each link. The authors of \cite{KWAS15,TZCPJ18} further design algorithms to match the regret lower bounds. Unfortunately, algorithms proposed for semi-bandit feedback setting cannot be extended to the bandit feedback setting as individual link feedback is not available.

When only end-to-end/bandit feedback is available, the authors of \cite{LZ12} propose algorithms with regret that has optimal dependence on the total number of rounds\footnote{The regret has sub-optimal dependence on the size of the network.}. But the algorithm requires the \dm~to enumerate over the path set to select path in each round. This degrades the practicality of the algorithms significantly, especially when deployed in large-scale networks. Existing works have also tried to investigate the problem through the more general linear stochastic bandits setting, see \eg, \cite{DHK08,AYAS09,AYPS11}. Nevertheless, the proposed algorithms again suffer from high computational complexity \cite{DHK08}. Even worse, existing works in linear stochastic bandits literature ignore the network structure of the action set. Hence, only sub-optimal regret bounds are achieved. 

 
As a matter of fact, the problem of stochastic online shortest path routing with end-to-end feedback falls into the category of combinatorial stochastic bandits, a special case of linear stochastic bandits with action set constrained to be subset of $\{0,1\}^d.$ However, finding efficient algorithms for combinatorial/linear stochastic bandits with (nearly) optimal regret remains as an open problem \cite{B16}. All of the above mentioned findings motivate us to exploit the networked structure of the action set to design efficient algorithms for the stochastic online shortest path problem with end-to-end feedback. Specifically, we aim at answering the following question:

\emph{Can we leverage the power of the network structure to design efficient algorithms that achieve (nearly) optimal instance-dependent and worst case regret bounds simultaneously for stochastic online shortest path routing under bandit-feedback?}

In this paper, we give an affirmative answer to the above question. We start with algorithms for the stochastic online shortest path routing problem with identifiable network structure, and gradually remove the extra assumptions to arrive at the most general case. Specifically, our contributions can be summarized as follows:
\begin{itemize}
	\item Assuming network identifiability, we first develop an efficient non-adaptive exploration algorithm with nearly optimal instance-dependent regret and sub-optimal worst case regret when the minimum gap \footnote{The concepts of network identifiability, instance-dependent regret, worst case regret, and minimum gap will be defined in Section \ref{sec:model} and \ref{sec:exploration_basis}.} is known.
	\item The main contribution is an adaptive exploration algorithm with nearly optimal instance-dependent regret without any knowledge of the minimum gap. Coupled with the novel Top-Two Comparison technique, the algorithms can be efficiently implemented. We also propose a simple modification for the algorithm to achieve nearly optimal worst case regret simultaneously.
	\item Complemented with an algorithm for finding basis in general networks, we show that our results can be applied to general networks without degrading the regret performances.
	\item We conduct extensive numerical experiments to validate  that our proposed algorithms not only achieve superior regret performances, but also reduce the runtime drastically.
\end{itemize}
The rest of the paper is organized as follows. In Section \ref{sec:model}, we describe the model of stochastic online shortest path routing with end-to-end feedback. In Section \ref{sec:exploration_basis}, we review the concepts of efficient exploration basis and make connections to network identifiability. Assuming network identifiability in Section \ref{sec:etc}, we propose the non-adaptive Explore-then-Commit algorithm to achieve nearly optimal instance-dependent regret when the minimum gap is known. In Section \ref{ttc}, we present the novel Top-Two Comparison and modified Top-Two Comparison algorithms to achieve nearly optimal instance-dependent and worst case regrets without any additional knowledge. In Section \ref{sec:general_networks}, we further study the problem without network identifiability, and propose an efficient algorithm with nearly optimal instance-dependent regret. In Section \ref{sec:numerical}, we present numerical results to demonstrate the empirical performances of the proposed algorithms. In Section \ref{sec:related_works}, we review related works in the bandits literature. In Section \ref{sec:conclusion}, we conclude our paper.

\section{Problem Formulation}
\label{sec:model}
\subsection{Notation}
Throughout the paper, all the vectors are column vectors by default unless specified otherwise. We define $[n]$ to be the set $\{1,2,\ldots,n\}$ for any positive integer $n.$ We use $\|\bm x\|_p$ to denote the $\ell_p$ norm of a vector $\bm x\in\Re^d.$ To avoid clutter, we often omit the subscript when we refer to the $\ell_2$ norm. For a positive definite matrix $A\in R^{d\times d}$, we use $\|\bm x\|_A$ to denote the  matrix norm $\sqrt{\bm{x}^{\top}A\bm x}$ of a vector $\bm x\in\Re^d.$ We also denote $x\wedge y$ as the minimum between $x,y\in\Re.$ We follow the convention to describe the growth rate using the notations $O(\cdot),\Omega(\cdot),$ and $\Theta(\cdot).$ If logarithmic factors are further ignored, we use $\widetilde{O}(\cdot),\widetilde{\Omega}(\cdot),$ and $\widetilde{\Theta}(\cdot),$ respectively. 
\subsection{Model}
Given a directed acyclic network $G$, an online stochastic shortest path problem is defined by a $d$-dimensional unknown but fixed mean link delay vector $\bmu\in[0,\mu_{\max}]^d$, paths $\a_k=\left(a_{k,1},\ldots,a_{k,d} \right)^{\top}\in\A\subseteq\{0,1\}^d$  for $1\leq k\leq K=|\A|$, and noise terms $\eta_t$ for $1\leq t\leq T,$ where $k$ is the index for paths and $t$ is the index for rounds. Here, $\A$ is the set of all possible paths in $G,$ and for a path $\a_k\in\A,$ $a_{k,j}=1$ if and only if it traverses link $j.$ With some abuse of notation, we use $k$ and $\a_k$ interchangeably to denote path $\a_k,$ and we refer to $\A$ as both a set and a matrix. Routing a packet through path $\a_k$ in round $t$ incurs the delay
$L_{t,k}=\langle \a_k,\bmu\rangle+\eta_t.$ Following the convention of existing bandits literature \cite{AYPS11}, we assume that $\eta_t$ is conditionally $R$-sub-Gaussian, where $R\geq0$ is a fixed and known constant. Formally, this means
\begin{align*}
	\forall\alpha\in\Re\quad\E\left[\exp\left(\alpha\eta_t\right)|a_{I_1},\ldots,a_{I_{t-1}},\eta_1,\ldots,\eta_{t-1}\right]\leq\exp\left(\frac{\alpha^2R^2}{2}\right)
\end{align*} 
and $$\E\left[\eta_t|a_{I_1},\ldots,a_{I_{t-1}},\eta_1,\ldots,\eta_{t-1}\right]=0.$$ In each round $t,$ a \dm~ follows a routing policy $P$ to choose the path $\It$ to route the packet based on its past
selections and previously observed feedback. Here, we consider \emph{end-to-end} (bandit) feedback setting in which only the delay of the selected path is observable as a whole rather than the \emph{individual} (semi-bandit) feedback in which the delays of all the traversed links are revealed. We measure the performance of $P$ via expected regret against the optimal policy with full knowledge of $\bmu$
\begin{align}
	\nonumber\E\left[\text{Regret}_T(P)\right]=\E\left[\sum_{t=1}^{T}L_{t,\It}-\min_{k\in[K]}\sum_{t=1}^TL_{t,k}\right]=\sum_{t=1}^T\langle \a_{\It},\bm\mu\rangle-T\langle \a_*,\bm\mu\rangle,
\end{align}
where $\a_*=\argmin_{\a_k\in\A}\langle \a_k,\bm{\mu}\rangle$ is the optimal path. In this paper, we require that $\a_*$ is unique. For any path $\a_k\neq\a_*,$ we define $\Delta_k=\langle\a_k-\a_*,\bmu\rangle$ as the difference of expected delay, \ie, the gap, between $\a_k$ and $\a_*.$ The maximum and minimum of $\Delta_{k}$ over all $k\in[K]$ with $\a_k\neq\a_*$ are denoted as $\Delta_{\max}$ and $\Delta_{\min},$ and are referred to as the maximum and  minimum gap, respectively.Without loss of generality, we assume $\mu_{\max}=1$\footnote{We shall relax this in the numerical experiments in Section \ref{sec:numerical}.} so that each path's expected delay is within $[0,d]$ and hence,
\begin{align}
\label{eq:delta_max}
\Delta_{\max}\leq d.
\end{align}

As it is common in stochastic bandit learning settings \cite{ABF02,AYPS11}, we distinguish between two different regret measures, namely the \emph{instance-dependent} regret and the \emph{worst case} regret
\begin{itemize}
	\item \textbf{Instance-dependent regret:} A regret upper bound is called instance-dependent if it is comprised of quantities that only depend on $T,d,\Delta_{k}$'s, and absolute constants.
	\item \textbf{Worst case regret:} A regret upper bound is called worst case if it is comprised of quantities that only depend on to $T,d,$ and absolute constants.
\end{itemize}
It is commonly known that when $\Delta_k$'s are allowed in the regret expressions, the regret can fall into the $\log T/\Delta_{\min}$ regime \cite{ABF02}. But depending on the choice of $\bmu,$ $\Delta_{{\min}}$ can become extremely small for any given $T,d,$ and $R,$ and the instance-dependent regret guarantee becomes meaningless. We therefore have to turn to the worst-case regret bound. We note that the regret is given by the minimum of the instance-dependent regret and worst case regret. Hence, it is desirable to obtain computationally efficient algorithms that have good instance-dependent and worst case regrets at the same time. Denoting $l_0$ as the maximal length of all the paths, \ie, $l_0=\max_{\a\in\A}\|\a\|_1,$ the instance-dependent regret lower bound is unclear yet, but from the combinatorial semi-bandits setting \cite{KWAS15} where individual feedback is available additionally, we know that it is of order at least $\Omega(dl_0\ln T/\Delta_{{\min}})$ \cite{S16}; The tight worst case regret lower bound is $\Omega\left(\sqrt{l_0^{3}dT}\right)$ \cite{CHK17}. 
\subsection{Design Challenges and Solution Strategies}
Since the mean link delay vector $\bmu$ is unknown, and we only get to know the end-to-end delay of the chosen path in each round, the \dm~falls into the so called exploration-exploitation dilemma. On one hand, the \dm~needs to explore the network to acquire accurate estimate of the expected delay of each path; while on the other, it needs to exploit the path with least delay to ensure low regret. As our problem resembles the stochastic multi-armed bandits problem, there are at least two natural approaches to address it:
\begin{itemize}
	\item \textbf{Optimism-in-the-Face-of-Uncertainty (OFU):} Following this principle, the \dm~balances exploration and exploitation by optimistically choosing the action with lowest confidence bound, \ie, the empirical mean loss with the confidence interval subtracted. In \cite{DHK08,AYPS11}, this approach has been shown to work in the general linear stochastic bandits setting, yet as pointed out in Section \ref{intro}, a direct adoption of the OFU principle to our problem cannot work. First, it fails to capture the underlying network structure, and brings a sub-optimal $O\left(\ln T\left(\ln T+d\ln\ln T\right)^2/{\Delta_{\min}}\right)$ instance-dependent and $O(d\log T\sqrt{T})$ worst case regret bounds \cite{AYPS11}. Even worse, the practicality of the algorithm is hindered by the high computational complexity in choosing the path to route. Indeed, it has been shown in \cite{DHK08} that the algorithm for path selection is polynomial time equivalent to a NP-hard negative definite linearly constrained quadratic programming.
	\item \textbf{Explore-then-Exploit:} Instead of doing exploration and exploitation simultaneously, the \dm~can collect data to construct accurate estimates for all actions' losses by first performing uniform exploration over all possible actions, and eliminates an action whenever it is confident that this action is sub-optimal. This procedure runs until there is only one action left. It has been shown in \cite{AO10} that the adaptive exploration approach works well for the ordinary stochastic multi-armed bandits setting. A similar approach has been applied to our problem of interest by the authors of \cite{LZ12}, and they achieve a sub-optimal $O\left(d_0^3d\ln T/{\Delta_{\min}}\right)$ instance-dependent regret with an inefficient algorithm. Here $d_0$ is the rank of the path matrix $\A.$
\end{itemize}  
As it is unclear how to get the OFU approach to work efficiently in our setting, we adopt the explore-then-exploit approach here. An immediate difficulty in implementing this approach is that the \dm~cannot afford to uniformly explore exponentially many paths. It's thus of great importance to devise a way to efficiently collect data in the stochastic online shortest path routing setting.
\section{Exploration Basis}
\label{sec:exploration_basis}
In order to execute the uniform exploration efficiently, the \dm~relies on a basis for the network. Intuitively, a set $\B\subseteq\A$ is a basis for $\A$ if it ``spans'' the set $\A,$ \ie, each path in $\A$ can be expressed as a linear combination of the paths in $\B.$ If the \dm~is able to accurately estimates the delays of the basis paths, it can also construct accurate delay estimators for all the paths in $\A$ thanks to the linearity property. It is worth noting that the concept of exploration basis has been raised in adversarial linear bandits before \cite{AK04}, and we review it here as it is going to be useful for our problem.
\subsection{Barycentric Spanners and Network Identifiability}
\label{ib}
Note that we have several requirements for $\B.$ First of all, the paths of $\B$ should come from $\A,$ \ie, $\B\subseteq\A,$ so that the \dm~can select them. Next, the set $\B$ should span the original path set $\A,$ \ie, $\text{rank}(\B)=d.$ Finally, denote the paths in $\B$ as $\b_1,\ldots,\b_d,$ and suppose any path $\a\in\A$ can be expressed as a linear combination of paths of $\b_1,\ldots,\b_d,$  \ie, there exits $\bnu_{\a}\in\Re^d,$ such that 
\begin{align}
\a=\B\bnu_{\a}=\sum_{i=1}^d\nu_{\a,i}\b_i.
\end{align}
We require that the absolute value of any $\nu_{\a,i}$ is bounded by some (small) positive constant $S,$ \ie,
\begin{align}
\label{eq:ub_bs1}
\forall\a\in\A\forall i\in[d]\quad\nu_{\a,i}\leq S.
\end{align}
To see the rationale behind the last requirement, we decompose the estimation error on $\a$'s delay as follows:
\begin{align}
\label{eq:error_decomp}
\left|\langle \a,\hbmu-\bmu\rangle\right|=\left|\left\langle \sum_{i=1}^d\nu_{\a,i}\b_i,\hbmu-\bmu\right\rangle\right|=\left|\sum_{i=1}^d\nu_{\a,i}\left\langle \b_i,\hbmu-\bmu\right\rangle\right|.
\end{align}
Here $\hbmu$ is any estimate of $\bmu.$ From eq. (\ref{eq:error_decomp}), we can see that all the $\nu_{\a,i}$'s should have small absolute values as otherwise, even small estimation error can be scale up drastically by any $\nu_{\a,i}$ with large absolute values. To this end, we introduce the concept of \emph{barycentric spanner} introduced by the authors of \cite{AK04}:
\begin{definition}[Barycentric spanner \cite{AK04}]
	\textit{Let $\mathcal{W}$ be a vector space over the real numbers, and $\mathcal{W}_0\subseteq\mathcal{W}$ a subset whose linear span is a $d$-dimensional subspace of $\mathcal{W}.$ A set $X=\{\x_1,\ldots,\x_d\}\subseteq \mathcal{W}_0$ is a barycentric spanner for $\mathcal{W}_0$ if every $\x\in \mathcal{W}_0$ may be expressed as a linear combination of elements of $X$ using coefficients in $[-1, 1].$ $X$ is the $S$-approximate barycentric spanner if every $\x\in \mathcal{W}_0$ may be expressed as a linear combination of elements of $X$ using coefficients in $[-S,S].$}
\end{definition}
The authors of \cite{AK04} also presented a result regarding the existence and search of barycentric spanner.
\begin{proposition}[\cite{AK04}]
	\label{bc_prop}
	Suppose $\mathcal{W}_0\subseteq \Re^d$ is a compact set not contained in any proper linear subspace. Given an oracle for optimizing linear functions over $\mathcal{W}_0,$ for any $S>1$ we may compute a $S$-approximate barycentric spanner for $\mathcal{W}_0$ in polynomial time, using $O\left(d^2\log_S(d)\right)$ calls to the optimization oracle.
\end{proposition}

The authors of \cite{AK04} also present an algorithm for finding a $S$-approximate barycentric spanner for any $S>1$. For completeness of presentation, we include this in Appendix \ref{sec:barycentric_app}. The assumption stated in Proposition \ref{bc_prop} that the set $\mathcal{W}_0$ is not contained in any proper subspace is closely related to network identifiability. Informally, we say that a network $G$ with $d$ links is \emph{identifiable} if $\A,$ its set of paths, spans the space $\Re^d$. In Theorem 3.1 of \cite{MHLST13}, the authors showed that it is in general impossible for $G$ to be identifiable if all the paths in $\A$ originate from and end at the same pair of nodes, but Theorem 3.2 of \cite{MHLST13} also states that it is possible for a subgraph of  $G$ to be identifiable. For ease of our discussion, we call each of the links that is incident to either the source or the destination as an \emph{external link}, and all other links the \emph{internal links}. A network $G_0\subseteq G$ with both the source and destination nodes as well as all the external links of $G$ removed is called the \emph{internal network.} In Fig.~\ref{fig:overlay_eg}, links $(1,2),(1,4),(3,6),$ and $(5,6)$ are external links; while the rest are internal links. We can see that the internal network with node $2,3,4,5$ is identifiable as the paths $(2,3),(2,3,5),(2,5),(2,4,5),$ and $(4,5)$ span the space $\Re^5.$ To this end, we temporarily make the following additional assumption (to be relaxed in Section \ref{sec:general_networks})
\begin{assumption}
	\label{identifiability_assumption}
	The internal network of $G$ is identifiable, and the expected delays of all the external links are known a priori. To avoid clutters, we further assume that the expected delays of the external links are deterministically 0.
\end{assumption}
With some abuse of notation, $d$ refers to the number of internal links whenever Assumption \ref{identifiability_assumption} is imposed, and it is equal to $d_0.$  Given Proposition \ref{bc_prop} and Assumption \ref{identifiability_assumption}, the \dm~can pick a positive number $S~(>1)$ first, and then implement Algorithm \ref{alg:barycentric} in Appendix \ref{sec:barycentric_app} to identify in polynomial time the $S$-approximate barycentric spanner $\B,$ \ie, for any path $\a\in\A,$ there exists some $\bnu_{\a}\in[-S,S]^d,$ such that $\B\bnu_{\a}=\a.$ By the definition of $S$-approximate barycentric spanner, the maximal $\ell_2$ norm of $\bnu_{\a}$ over all $\a\in\A$ is upper bounded by $S\sqrt{d},$ \ie,
\begin{align}
\label{eq:ub_bs}
\max_{\a\in\A}\|\bnu_a\|\leq\sqrt{\sum_{i=1}^dS^2}\leq S\sqrt{d}.
\end{align}
\section{Explore-then-Commit Algorithm: A Warm-Up}
\label{sec:etc}
In this section, we develop the \emph{Explore-then-Commit} (EC) algorithm based on non-adaptive exploration to solve the problem.
\subsection{Design Intuitions}
The design of the \etc~follows an intuitive rationale: if the \dm~is able to recover the expected delay of each path of the $\B$ accurately, it will also be able to accurately estimate the expected delay of each path as the delay of each path is the linear combination of the elements in the barycentric spanner. Once the \dm~believes that the optimal path has been found with high probability, it could choose to commit to this path, and incurs low regret. To begin, we assume that the \dm~knows the minimum gap $\Delta_{\min}.$ We will later relax this assumption to obtain practical algorithms. 
\subsection{Design Details}
Given a positive integer $n~(\leq \lfloor T/d\rfloor),$ we aim at getting a good estimate of $\bmu$ in the first $n\cdot d$ rounds, and then chooses the estimated best path in each of the remaining $T-n\cdot d$ rounds. We thus call the first $n\cdot d$ rounds as the exploration stage, and the remaining $T-n\cdot d$ rounds as the committing stage. The \etc~divides the exploration stage into epochs of length $d,$ and chooses each path in $\B$ once in every epoch until the end of the exploration stage. Afterwards, the \etc~makes use of the \emph{Ordinary Least Square}~(OLS) estimator to construct an estimate for $\bmu.$ Specifically, the paths used in the first $n$ epochs (or $n\cdot d$ rounds) form the design matrix 
$$\D_{n}=\begin{pmatrix}
		\a_{I_1},\ldots,\a_{I_{nd}}
	\end{pmatrix}^{\top}$$
and the observed losses form the response vector
$$\r_{n}=\begin{pmatrix}L_{1,I_1},\ldots,L_{nd,I_{nd}}\end{pmatrix}^{\top}.$$
The OLS estimator then gives us \begin{align}\label{eq:ols}\hbmu_n=\left(D_n^{\top}D_n\right)^{-1}D_n^{\top}\r_n.\end{align}
 Thanks to the identifiability assumption, $D_n^{\top}D_n$ is full rank, and $\hbmu_n$ is well-defined. One can easily verify 
$\E\left[\hbmu_n\right]=\bmu.$
Finally, the \etc~applies an arbitrary shortest path algorithm to compute the path with the lowest estimated delay, and commits to this path in the exploitation stage.
\subsection{Regret Analysis}
To properly tune the parameter $n,$ an essential tool is a deviation inequality on the estimation errors. 
\begin{theorem}
	\label{ad}
	After $m$ epochs of explorations, the probability that there exists a path $\a\in\A,$ such that the estimated mean delay of $\a$ deviates from its mean delay by at least $SR\sqrt{{2\ln (2)d^2+4d\ln{\delta^{-1}}}/{m}}$ is at most $\delta,$ \ie,
	\begin{align*}
	\Pr\left(\exists\a\in\A:\left|\langle \a,\bmu\rangle-\langle \a,\hbmu_m\rangle\right|\geq SR\sqrt{\frac{2\ln (2)d^2+4d\ln{\delta^{-1}}}{m}}\right)\leq\delta.
	\end{align*}
\end{theorem}
\begin{proof}
	The proof of Theorem \ref{ad} makes use of the convergence property of the OLS estimator and the fact that the $\B$ is the $S$-approximate barycentric spanner with $S=2$. Please refer to Section \ref{sec:ad} for the complete proof.
\end{proof}
 We are now ready to present the regret bound of \etc.
\begin{theorem}
	\label{etcthm}
	With the knowledge of $\Delta_{\min},$ \etc~has the following regret bounds:
	\begin{itemize}
		\item \textit{Instance-dependent regret:}
		$O\left(\frac{\left(d^2\ln (dT)+d^3\right)\Delta_{\max}}{\Delta_{\min}^2}\right)$
		\item \textit{Worst case regret:} $\widetilde{O}\left(d^{\frac{4}{3}}T^{\frac{2}{3}}\right).$
	\end{itemize}
\end{theorem}
\begin{proof}
	Please refer to Section \ref{sec:etcthm} for the complete proof.
\end{proof}
\begin{remark} 
	The instance-dependent regret bound obtained in Theorem~\ref{etcthm} is a significant improvement compared to the direct application of OFU approach, and the worst case regret can be achieved without knowing $\Delta_{\min}.$ Nevertheless, we should be aware that the choice of $n$ for the instance-dependent regret bound relies on knowing $\Delta_{\min},$ which is never the case in practice. 
\end{remark}
Though being computationally efficient, the above remark indicates that the non-adaptive \etc~is not sufficient to achieve optimal regret bounds.
\section{Top-Two Comparison Algorithm: An Adaptive Exploration Approach}
\label{ttc}
As we have seen from the previous discussions, the non-adaptive \etc~fails to make full use of the observed delays to explore adaptively, and its success relies almost solely on knowing $\Delta_{\min}$ ahead of time. 

In this section, we study adaptive exploration algorithms that have been shown to achieve nearly optimal regret bounds in stochastic MAB \cite{AO10,S17} to obtain nearly optimal instance-dependent and worst case regret bounds. Different from those in ordinary stochastic MAB settings, the algorithm builds on top of a novel \emph{top two comparison} (TTC) method to allow efficient computation. We start by attaining a nearly optimal instance-dependent regret bound, and then show how to attain a nearly optimal worst case regret bound simultaneously.
\subsection{Design Intuitions}
\label{ttc:intuition}
Adaptive exploration algorithms often serve as an alternative for UCB algorithms in stochastic multi-armed bandits  \cite{AO10,S17}. In \cite{AO10,S17}, the \dm~uniformly explores all remaining actions, and periodically executes an action elimination rule to ensure with high probability that:
\begin{itemize}
	\item The optimal action remains;
	\item The sub-optimal actions can be removed effectively.
\end{itemize}
until only one action is left, and commits to that action in the rest of the rounds. The adaptive exploration algorithms achieve optimal $O(K\log T)$ instance-dependent and $O(\sqrt{KT\log T})$ worst case regret bounds for stochastic multi-armed bandits.

We start by demonstrating how an adaptive exploration algorithm can achieve the nearly optimal $O\left(\left(d\log T+d^2\right)\Delta_{\max}/\Delta_{\min}\right)$ instance-dependent regret bound. Similar to the \etc, the adaptive exploration algorithm also splits the $T$ rounds into an exploration stage and a committing stage: in each epoch $m=1,2,\ldots$ of the exploration stage, the \dm~selects every path in $\B$ once so that all of them have $m$ samples. To ease our presentation, we denote the estimated shortest path after $m$ epochs of uniform exploration as $\tilde{a}_m,$ \ie,
\begin{align*}
\tilde{\a}_m\leftarrow\argmin_{\a\in\A}\langle\a,\hbmu_{m}\rangle,
\end{align*} 
and follow Theorem \ref{ad} to denote the $1-\delta$ confidence bound as $\tilde{\Delta}_m,$ \ie,
\begin{align}\label{eq:conf_interval}\tilde{\Delta}_m=SR\sqrt{\frac{2\ln (2)d^2+4d\ln{\delta^{-1}}}{m}}.\end{align}
We denote the total length of exploration stage by a random variable $N.$ We then use a simple union bound to show the probability that there exists a path $\a\in\A,$ such that the estimated mean delay of $\a$ deviates from its mean delay by at least $\tilde{\Delta}_m$ at the end of any epoch in the committing stage can be upper bounded as 
\begin{align}
\nonumber&\Pr\left(\exists m\in\left[N\right],\a\in\A:\left|\langle \a,\bmu\rangle-\langle \a,\hbmu_m\rangle\right|\geq \tilde{\Delta}_m\right)\\
\nonumber\leq&\sum_{m=1}^{N}\Pr\left(\exists \a\in\A:\left|\langle \a,\bmu\rangle-\langle \a,\hbmu_m\rangle\right|\geq \tilde{\Delta}_m\right)\\
\label{union_bound}\leq&\sum_{m=1}^{T}\delta\\
\nonumber\leq&\frac{T\delta}{d},
\end{align}
where we have used Theorem \ref{ad} and the fact that $N\leq T$ in inequality (\ref{union_bound}). In other words, if we denote the event $E$ as following: any path $\a_k$'s estimated delay $\langle\a_k,\hbmu_{m}\rangle$ is within $\tilde{\Delta}_m$ distance from its true expected delay $\langle\a_k,\bmu\rangle$ for all $m\in[N],$ \ie,
\begin{align}
E=\{\forall m\in[N]\forall \a\in\A:\left|\langle \a,\bmu\rangle-\langle \a,\hbmu_m\rangle\right|\leq \tilde{\Delta}_m\}
\end{align}
then event $E$ holds with probability at least $(1-T\delta/d)$ in the adaptive exploration algorithm. From inequality (\ref{eq:delta_max}), we have $\Delta_{\max}\leq d,$ and the worst possible total regret (\ie, choosing the path with maximum gap in each round) an algorithm can incur is $T\Delta_{\max}\leq Td,$ we can tune $\delta$ properly, \ie, setting $\delta=T^{-2},$ so that the regret incurred by the algorithm in case $E$ does not hold is at most $1.$  Therefore, we only need to focus the case when $E$ holds.

Conditioned on $E,$ we assert that the \dm~could detect if any of the remaining paths $\a_k$ is sub-optimal by checking whether
\begin{align}
\label{checking}
\langle\a_k,\hbmu_{m}\rangle-\langle\tilde{\a}_m,\hbmu_m\rangle>2\tilde{\Delta}_m
\end{align} 
holds at the end of each epoch $m.$ Afterwards, the identified sub-optimal paths are eliminated. We use Figure \ref{fig:ttc_graph} to illustrate the rationale behind this criterion. Note that in both Fig. \ref{fig:ttc_graph1} and \ref{fig:ttc_graph2}, the horizontal right arrow is the positive number axis. 

In Fig. \ref{fig:ttc_graph1}, suppose $\langle\tilde{\a}_m,\hbmu_m\rangle$ and $\langle\a_k,\hbmu_{m}\rangle$ lie at $B$ and $F,$ respectively. Conditioned on event $E,$ $\langle\tilde{\a}_m,\bmu\rangle$ should locate within the interval $[A,C]$ while $\langle\a_k,\bmu\rangle$ should locate within the interval $[D,H].$ Now if $B$ and $F$ are more than $2\tilde{\Delta}_m$ away from each other, then \begin{align}\langle\tilde{\a}_m,\bmu\rangle<\langle\a_k,\bmu\rangle.\end{align} 
In other words, path $\a_k$ is sub-optimal as its expected delay is at least longer than $\tilde{\a}_m.$

Similarly in Fig. \ref{fig:ttc_graph2}, suppose $\langle\tilde{\a}_*,\bmu\rangle$ and $\langle\a_k,\bmu\rangle$ lie at $A'$ and $D',$ respectively. Conditioned on event $E,$ $\langle\tilde{\a}_m,\hbmu_m\rangle~(\leq\langle{\a}_*,\hbmu_m\rangle)$ should locate to the left of $B'$ while $\langle\a_k,\hbmu_m\rangle$ should locate to the right of $C'.$ Now if $\Delta_k>4\tilde{\Delta}_m,$ then
\begin{align}
\langle\a_k,\hbmu_{m}\rangle-\langle\tilde{\a}_m,\hbmu_m\rangle>2\tilde{\Delta}_m,
\end{align}
which means the sub-optimal path $\a_k$ is detected according to criterion (\ref{checking}).
\begin{figure}[!ht]
	\subfigure[The removed arm is sub-optimal]{\label{fig:ttc_graph1}\includegraphics[width=8cm,height=3.8cm]{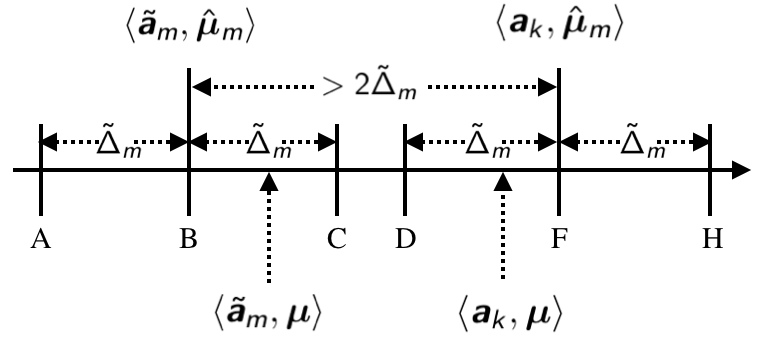}}
	\subfigure[All sub-optimal actions can be detected effectively]{\label{fig:ttc_graph2}\includegraphics[width=8cm,height=3.8cm]{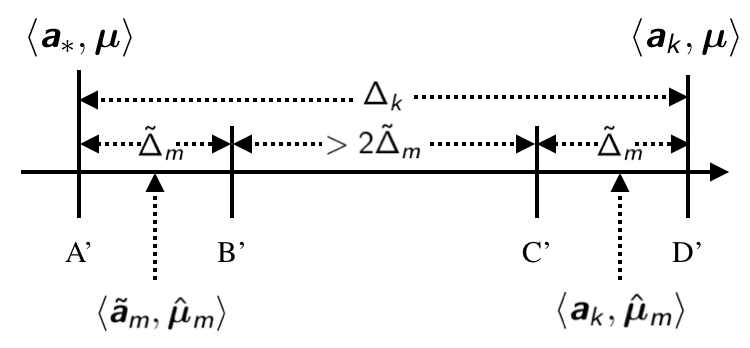}}
	\caption{Intuitions underpinning criterion (\ref{checking})}
	\label{fig:ttc_graph}
\end{figure}

We formalize these observations in the following lemma.
\begin{lemma}
	\label{lemma:ttc_intuition}
	Conditioned on event $E,$  if criterion (\ref{checking}) holds, then
	\begin{enumerate}
		\item path $\a_k$ is sub-optimal;
		\item any sub-optimal path $\a_k$ with $\Delta_k>4\tilde{\Delta}_m$ is detected.
	\end{enumerate}
\end{lemma}
\begin{proof}
	The proof follows from the above arguments. Please refer to Section \ref{sec:lemma:ttc_intuition} for the complete proof.
\end{proof}
These two nice properties of criterion (\ref{checking}) jointly guarantee that the optimal path remains in $\A$, and any sub-optimal path $\a_k$ is removed once $\tilde{\Delta}_m$ shrinks down to below $\Delta_{\min}/4.$ Specifically, if $m$ arrives to a value $\overline{m}$ that $\tilde{\Delta}_{\overline{m}}\leq\Delta_{\min}/4$ or 
\begin{align}\overline{m}=\frac{S^2R^2(32\ln (2)d^2+64d\ln{\delta^{-1}})}{\Delta^2_{\min}}\end{align}
according to eq. (\ref{eq:conf_interval}), all sub-optimal paths should have been eliminated.

Roughly speaking, conditioned on $E,$ the regret of the adaptive algorithm is 
\begin{align}d\overline{m}\Delta_{\max}=O\left(\frac{\left(d^2\ln{\delta^{-1}}+d^3\right)\Delta_{\max}}{\Delta^2_{\min}}\right).\end{align}
Recalling that the regret conditioned on $\neg E$ is at most $Td,$ setting $\delta=T^{-2},$ the expected regret of this algorithm is upper bounded as
$O\left({\left(d^2\ln{T}+d^3\right)\Delta_{\max}}/{\Delta^2_{\min}}\right),$ and we shall formalize this analysis in Theorem \ref{ttcthm}. Surprisingly, adaptivity saves us from a lack of knowledge on the exact value of $\Delta_{\min}.$
\subsection{Efficient Implementation}
\label{ttc:implementation}
One may note that implementing the criterion (\ref{checking}) requires an enumeration over the set $\A,$ which is typically exponential in size (in terms of $d$). In this subsection, we further propose a polynomial time implementation, namely the \emph{Top Two Comparison} (TTC) algorithm, for our problem.

Different from the adaptive exploration algorithms proposed for stochastic multi-armed bandit problems \cite{AO10,S17}, which uniformly explores the set of remaining actions, our strategy decouples the exploration basis $\B$ from path elimination by making use of the $S$-approximate barycentric spanner $\B.$ In other words, the \dm~does not need to eliminate the sub-optimal paths one by one. As the optimal path is unique by assumption, it can instead remove all of them at the same time once the difference between the delay of the estimated shortest path and the delay of the estimated second shortest path is larger than $2\tilde{\Delta}_m$ for some epoch $m.$ 

To find the estimated second shortest path, we make the observation that the estimated second shortest path should traverse at least one link that is different than those in the estimated shortest path. The \dm~could start by iteratively setting the delay of links traversed by the shortest path to a large number, \ie, $100d,$ one at a time, while keeping the estimated delays of all other links intact, and find the delay of the shortest path with respect to the ``perturbed" estimated delay vector. Finally, the minimum delay over these ``perturbed" delays is the second shortest delay. 

\subsection{Design Details}
\label{ttc:detail}
We are now ready to formally present the \ttc. Following the design guidelines presented in Sections \ref{ttc:intuition} and \ref{ttc:implementation}, the \ttc~initializes the set of remaining paths as $\A_1=\A,$ and divides the time horizon into epochs. In the $m^{\text{th}}$ epoch, \ttc~distinguishes two cases:
\begin{enumerate}
	\item If $\A_{m}$ contains only one path, \ttc~chooses this path, and sets $\A_{m+1}=\A_m;$
	\item Otherwise, the \ttc~picks each path in $\B$ once so that every path in $\B$ has been selected $m$ times. It then computes the OLS estimate $\hbmu_{m}$ for $\bmu,$ and identifies the path $\tilde{\a}_m$ with least estimated delay, \ie, $\tilde{\a}_m=\argmin_{\a\in\A_m}\langle\a,\hbmu_{m}\rangle$ and the path with estimated second  shortest delay, \ie,  $\overline{\a}_m=\argmin_{\a\in\A_m\setminus\{\tilde{a}_m\}}\langle\a,\hbmu_{m}\rangle$ via a second shortest path sub-routine. Afterwards, \ttc~checks the gap between $\tilde{\a}_m$ and $\overline{\a}_m:$ If $\langle\overline{\a}_m,\hbmu_{m}\rangle-\langle\tilde\a,\hbmu_{m}\rangle\geq2\tilde\Delta_m.$ The set of remaining paths for the $(m+1)^{\text{th}}$ epoch is denoted as ${\A}_{m+1}=\{\tilde{\a}_m\};$ otherwise, $\A_{m+1}=\A_m.$
\end{enumerate}
The pseudo-code of \ttc~is shown in Algorithm \ref{ttc:alg} and the pseudo-code of the sub-routine for finding second shortest path is shown in Algorithm \ref{ttc:sub-routine}. Please note that the algorithms are run in epochs (indexed by $m$), and $\A$ can be represented by the incidence matrix of $G.$
\begin{algorithm}[!ht]
	\caption{Top-Two Comparison Algorithm}
	\label{ttc:alg}
	\begin{algorithmic}[1]
		\State \textbf{Input:} A set of paths $\A,$ a $S$-approximate barycentric spanner $\B\subseteq\A,$ time horizon $T.$
		\State \textbf{Initialization:} $\A_1\leftarrow\A,\tilde{\Delta}_m\leftarrow SR\sqrt{(2\ln (2)d^2+8d\ln{T}/m}$ for $m=1,2,\ldots.$
		\For{epoch $m=1,2,\ldots$}
		\If{$|\A_m|=1,$}
		\State Choose the path in $\A_m.$
		\State $\A_{m+1}\leftarrow\A_m.$
		\Else
		\State Choose each path in $\B$ once.
		\State $\hbmu_{m}\leftarrow\left(\D_{m}^{\top}\D_{m}\right)^{-1}\D_{m}^{\top}\r_{m}.$
		\State 
		$\tilde{\a}_m\leftarrow\argmin_{\a\in{\A}_m}\langle\a,\hbmu_{m}\rangle.$
		\State $\overline{\a}_m\leftarrow \textnormal{SSP}(\A,\hbmu_m,\tilde{\a}_m)$\quad(calls the second shortest path sub-routine).
		\If{$\langle\overline{\a}_i,\hbmu_{m}\rangle-\langle\tilde{\a}_i,\hbmu_{m}\rangle>2\tilde{\Delta}_m,$}
		\State $\A_{m+1}\leftarrow\{\tilde{\a}_m\}.$
		\Else
		\State $\A_{m+1}\leftarrow\A_m.$
		\EndIf
		\EndIf
		\EndFor
	\end{algorithmic}
\end{algorithm}
\begin{algorithm}[!ht]
	\caption{Second Shortest Path (SSP) Sub-Routine }
	\label{ttc:sub-routine}
	\begin{algorithmic}[1]
		\State \textbf{Input:} A set of paths $\A,$ a vector $\bm{\psi}$ of link delays, and the shortest path $\a$ with respect to $\A$ and $\bm{\psi}.$
		\State \textbf{Output:} The second shortest path with respect to $\A$ and $\bm{\psi}.$
		\State \textbf{Initialization: $\bm{s}\leftarrow\bm{0}^d.$}
		\For{all $j\in[d]$}
		\If{$a_{j}=1$}
		\State $\bm{\psi}_j'\leftarrow\bm{\psi},\psi'_{j,j}\leftarrow 10d.$
		\State $\bm{c}_j\leftarrow\argmin_{\a'\in\A}\langle\a',\bm{\psi}'_j\rangle.$
		\State $s_{j}\leftarrow\langle\bm{c}_j,\bnu\rangle.$
		\EndIf
		\EndFor
		\State $j'=\argmin_{j\in[d]:a_j=1}s_{j}.$
		\State \Return $\bm{c}_{j'}.$
	\end{algorithmic}
\end{algorithm}
\subsection{Regret Analysis}
The analysis essentially follows the intuition presented in Section \ref{ttc:intuition}, and the instance-dependent regret of the \ttc~is given by the following theorem.
\begin{theorem}
	\label{ttcthm}
	For any $T\geq d\overline{m}={dS^2R^2(32\ln (2)d^2+128d\ln{T})}/{\Delta^2_{\min}},$ the instance-dependent expected regret of \ttc~is bounded as
	\begin{align*}
		&\E\left[\textnormal{Regret}_T\left(\text{\ttc}\right)\right]\leq O\left(\frac{\left(d^2\ln{T}+d^3\right)\Delta_{\max}}{\Delta^2_{\min}}\right).
	\end{align*}
\end{theorem}
\begin{proof}
	Please refer to Section \ref{sec:ttcthm} for the complete proof.
\end{proof}
We now comment on the bound provided in Theorem \ref{ttcthm}. In the worst case, \ie, when $\Delta_{\max}=d,$ if $\Delta_{\min}\leq d^{3/2}{T^{-1/4}},$  the RHS of Theorem \ref{ttcthm} is of order $\widetilde{\Omega}\left(d\sqrt{T}+d^2\right).$ As the regret bound from adversarial linear bandits is of order $\widetilde{O}(d\sqrt{T}),$ this indicates that the instance-dependent regret bound becomes meaningless once $\Delta_{\min}$ becomes smaller than $d^{3/2}T^{-1/4}.$ Even though adaptive exploration saves us from not knowing $\Delta_{\min},$ it cannot achieve nearly optimal worst case regret bound automatically. This is because the \ttc~shares similar structure to \etc, and as we have seen in Theorem \ref{etcthm} that tuning the parameter $n$ to achieve sub-optimal $\widetilde{O}\left(d^{4/3}T^{2/3}\right)$ worst case regret bound does not require any knowledge of $\Delta_{\min},$ either. Some other techniques are needed if we want to get nearly optimal instance-dependent and worst case regrets at the same time.
\subsection{Getting Nearly Optimal Worst Case Regret}
It turns out that we can get nearly optimal instance-dependent and worst case regrets at the same time with just a bit more effort. The key idea is to limit the length of the exploration stage so that once the smallest gap $\Delta_{\min}$ is believed to be smaller than $dT^{-1/4}$ with high probability, the \dm~switches to an efficient alternative algorithm for adversarial linear bandits to solve the problem. A candidate for the alternative algorithm can be found in \cite{BE15}. Specifically, we set $$\overline{n}=\sqrt{T}S^2R^2(2\ln (2)d+8\ln{T})/d^2,$$ and modify the \ttc~as following:
\begin{enumerate}
	\item For each epoch $m\leq\overline{n},$ the \dm~runs the \ttc;
	\item If the set $\A_{\overline{n}+1}$ contains only one path, the \dm~selects this path in the rest of the rounds;
	\item Else if the set $\A_{\overline{n}+1}$ contains more than one path, the \dm~finds that $\Delta_{\min}\leq 4\tilde{\Delta}_{\overline{n}}=\widetilde{O}\left(d^{3/2}{T^{-1/4}}\right)$
	holds with probability at least $(1-1/(dT)),$ and thus
	terminates the \ttc, and runs the efficient algorithm for adversarial linear bandits in \cite{BE15} over the network to solve the problem.
\end{enumerate}
 We name this as the \emph{Modified Top Two Comparison} (MTTC) algorithm, and its pseudo-code is shown in Algorithm \ref{mttc:alg}.
 \begin{algorithm}[!ht]
 	\caption{Modified Top-Two Comparison Algorithm}
 	\label{mttc:alg}
 	\begin{algorithmic}[1]
 		\State \textbf{Input:} A set of paths $\A,$ a $S$-approximate barycentric spanner $\B\subseteq\A,$ time horizon $T.$
 		\State \textbf{Initialization:} $\A_1\leftarrow\A,\overline{n}\leftarrow\sqrt{T}S^2R^2(2\ln (2)d+8\ln{T})/d^2,\tilde{\Delta}_m\leftarrow SR\sqrt{(2\ln (2)d^2+8d\ln{T}/m}$ for $m=1,2,\ldots$
 		\For{epoch $m=1,2,\ldots,\overline{n}$}
 		\State{Run \ttc}
 		\EndFor
 		\If{$|\A_{\overline{n}+1}|=1$}
 		\State Choose the path in $\A_{\overline{n}+1}$ for the rest of the rounds.
 		\Else
 		\State Run the efficient algorithm for adversarial linear bandits in \cite{BE15}.
 		\EndIf
 	\end{algorithmic}
 \end{algorithm}
We are now ready to state the regret bound of \mttc.
\begin{theorem}
	\label{mttcthm}
	For any $T\geq d\overline{n},$ the \mttc~has the following regret bounds:
		\begin{itemize}
		\item \textit{Instance-dependent regret:} $$O\left(\frac{\left(d^2\ln{T}+d^3\right)\Delta_{\max}}{{\Delta}_{\min}^2}\right).$$
		\item \textit{Worst case regret:} $$\widetilde{O}\left(d\sqrt{T}\right).$$
	\end{itemize}
\end{theorem}
\begin{proof}
	Please refer to Section \ref{sec:mttcthm} for the complete proof.
\end{proof}
\section{General Networks}
\label{sec:general_networks}
The success of the \ttc~and the \mttc~in achieving nearly optimal regrets relies on the identifiability assumption, \ie, Assumption \ref{identifiability_assumption}, which might be violated in practice. For example, if the network scale grows large, it is very likely that even the internal network of $G$ is not fully identifiable. Also, if the external links are shared among many entities, it is hard to obtain the expected delays of all the external links. For a general network, one possible way to find a $S$-approximate barycentric spanner is to project $\A$ into some sub-space so that it is still full rank in that sub-space.  But it is unclear how to implement the projection without enumerating all the paths in $\A,$ which is computationally inefficient. Therefore, we are in need of a new technique for our problem. In this section, we show how to implement the \mttc~algorithm for general networks. We start by proposing an algorithm for finding a basis $\B$ of $\A$ when $\A$ does not span $\Re^d$. We note that any basis of $\B$ is automatically $S$-approximate barycentric spanner of $\A$ with some (possibly unknown at first) positive number $S.$ We then state the difference in estimating $\bmu$ between identifiable and general networks, and present a general version of OLS estimator with provable deviation property. Finally, we present an algorithm for determining $S.$ Throughout this section, we shall assume that the rank of $\A$ is $d_0<d.$ 
\subsection{Additional Notation}
\label{sec:matrix_notation}
In this section, we will make use of matrix notations heavily. For any matrix $M\in\Re^{d_1\times d_2},$ we use $M(i,j)$ to denote its element at the $i^{\text{th}}$ row and $j^{\text{th}}$ column, $M(i,:),$ and $M(:,j)$ to denote its $i^{\text{th}}$ row and $j^{\text{th}}$ column vectors, respectively, and $M\left([i_1,i_2],:\right),$ and $M\left(:,[j_1,j_2]\right)$ to denote the matrices obtained by keeping only the $i_1^{\text{th}}$ to $i_2^{\text{th}}$ rows and $j_1^{\text{th}}$ to $j_2^{\text{th}}$ columns, respectively Moreover, $M(-i,:)$ and $M(:,-j)$ are the matrices obtained by removing the $i^{\text{th}}$ row and $j^{\text{th}}\text{ column}$ of $M,$ respectively. $M(-i,-j)$ is the $(d_1-1)$-by-$(d_2-1)$ matrix obtained by removing the $i^{\text{th}}$ row and $j^{\text{th}}$ column of $M$ simultaneously.
\subsection{Efficient Algorithm for Finding the Basis}
\label{sec:general_networks_basis}
As a first step, we present a greedy algorithm that finds the basis $\B$ of $\A$ even when the network $G$ is unidentifiable. Inspired by the algorithm for finding the $S$-approximate barycentric spanner for identifiable networks, \ie, Algorithm \ref{alg:barycentric} in Appendix \ref{sec:barycentric_app}, the high-level idea of the algorithm can be described as following:
\begin{enumerate}
	\item Initiate a matrix $\C$ to the $d$-by-$d$ identity matrix;
	\item Greedily replace as many columns of $\C$ as possible by paths in $\A$ while keeping $\C$ full rank. 
	\item All the columns in $\C$ that are obtained from $\A$ constitute $\B.$ 
\end{enumerate}  
Since steps (1) and (3) can be easily implemented, we further elaborate on an iterative algorithm for step (2). For ease of presentation, we use $\C_u$ to denote the resulted matrix after the $u^{\text{th}}$ iteration with $\C_0=\C.$ At the beginning of the $(u+1)^{\text{th}}$ iteration, suppose $\C_u$ can be written as 
\begin{align}\label{eq:gn2}\C_u=\left(\C'_u,\C''_u\right),\end{align} 
where $\C'_u$ are the columns obtained from $\A;$ while $\C''_u$ are the columns inherited from $\C_0$, the algorithm then finds a column $\c\in\C''_u$ such that replacing $\c$ with an element in $\a\in\A$ can result in a full rank matrix, and sets \begin{align}\label{eq:gn}\C_{u+1}=\left(\a,\C_u(:,-j)\right),\end{align} where $j$ is the column index of $\c.$ This algorithm terminates once such $\c$ cannot be found in $\C_u$ after some iterations $u.$

To efficiently implement the above iterative algorithm, \ie, to find such $\a$ in each iteration if it exists, we note that the matrix $\C_{u+1}$ is full rank if and only if the determinant of $\C_{u+1}$ is nonzero, \ie,
\begin{align}\text{rank}(\C_{u+1})=d \quad\Leftrightarrow\quad \det\C_{u+1}\neq 0.\end{align}
For now, suppose we are given a full rank matrix $\C_u,$ if the $j^{\text{th}}$ column of $\C_u$ is replaced by an $\a\in\A$ to form $$\C_u^j=\left(\C_u(1,j-1),\a,\C_u(j+1,d)\right),$$ the determinant of $\C_u^j$ can be written as a linear function of $\a,$ \ie, 
\begin{align}
\label{eq:gn1}
\det\C_u^j=\sum_{i=1}^d\left[(-1)^{i+j}\det \left(\C_u(-i,-j)\right)\right]a_i
\end{align}
by the Laplace expansion, and the value of $\det \left(\C_u(-i,-j)\right)$ can be computed efficiently using the LU decomposition. Now to find an index $j (>u)$ and $\a$ that satisfies $\det\C_u^j\neq0,$ we can equivalently solve the following optimization problem \begin{align}\label{opt:gn}\max_{\a\in\A}\left|\det\C_u^j\right|,\end{align} 
for all $j>u.$ If there exists some $j>u$ such that the solution $\a$ satisfies $\left|\det\C_u^j\right|>0,$ we can then replace the $j^{\text{th}}$ column of $\C_u$ by $\a$ to form $\C_{u+1}$ according to eq. (\ref{eq:gn}). 

For a given $j,$ defining a vector $\c_j\in\Re^d$ with each entry defined by eq. (\ref{eq:gn1}), \ie,
\begin{align}\forall i\in[d]\quad c_{j,i}=\left[(-1)^{i+j}\det \left(\C_u(-i,-j)\right)\right],\end{align}  
the optimal solution of (\ref{opt:gn}) can be obtained by first solving the following two sub-problems
\begin{align}
\max_{\a\in\A}\langle\c_j,\a\rangle,\quad\min_{\a\in\A}\langle-\c_j,\a\rangle,
\end{align}
and then picking the solution with larger absolute value. To solve the first sub-problem, we can use the following steps: 
\begin{enumerate}
	\item Assign delay $c_{j,i}$ to link $i$ of $G$ for all $i\in[d];$
	\item Compute the longest path. This requires a call to an appropriate efficient longest path algorithm for directed acyclic graphs, \eg, topological sorting \cite{CLRS09}.
\end{enumerate} 
The formal description of this algorithm for basis identification is shown in Algorithm \ref{general_network:alg}.
 \begin{algorithm}[!ht]
	\caption{Basis Identification for General Networks}
	\label{general_network:alg}
	\begin{algorithmic}[1]
		\State \textbf{Input:} A set of paths $\A.$
		\State \textbf{Initialization:} $\C_0\leftarrow I,u\leftarrow0,$Flag$\leftarrow$True.
		\State \textbf{Output:} $\B,$ the basis $\A.$
		\While{$u\leq d-1$ \textbf{and} Flag$==$True}
		\For{$j=u+1,\ldots,d$}
		\State{$\forall i\in[d]~c_{j,i}\leftarrow(-1)^{i+j}\det \left(\C_u(-i,-j)\right).$}
		\State{$\a'_1\leftarrow\argmax_{\a\in\A}\langle\c_j,\a\rangle.$}
		\State{$\a'_2\leftarrow\argmin_{\a\in\A}\langle-\c_j,\a\rangle.$}
		\State{$\a\leftarrow\argmax_{\a'_1,\a'_2}\left\{|\langle\c_j,\a'_1\rangle|,|\langle\c_j,\a'_2\rangle|\right\}$}
		\If{$|\langle\c_j,\a\rangle|>0$}
		\State{$\C_{u+1}\leftarrow\left(\a,\C_u(:,-j)\right).$}
		\State{$u\leftarrow u+1.$}
		\State{\textbf{break}}
		\Else{\textbf{ if }$j==d$}
		\State{Flag$\leftarrow$False.}
		\EndIf
		\EndFor
		\EndWhile
		\State{$\B\leftarrow\C_{u}(:,[1:u])$}\\
		\Return{$\B.$}
	\end{algorithmic}
\end{algorithm}

We are now ready to prove the correctness of the algorithm, \ie, if the rank of $\A$ is $d_0<d,$ then Algorithm \ref{general_network:alg} returns a basis $\B\subseteq\A,$ such that the rank of $\B$ is $d_0.$
\begin{lemma}
	\label{general_nework:correctness}
	Algorithm \ref{general_network:alg} terminates in polynomial time. Upon termination, the matrix $\B$ returned by Algorithm \ref{general_network:alg} is a basis of $\A,$ \ie, $\B$ has linearly independent columns and for every $\a\in\A,$ there exists a vector $\bnu_a,$ such that $\B\bnu_a=\a.$
\end{lemma}
\begin{proof}
	Please refer to Section \ref{sec:general_nework:correctness} for the complete proof.
\end{proof}
\begin{remark}
	\label{remark:S}
Although $\A$ does not span $\Re^d,$ we still develop an efficient algorithm for computing the basis of $\A.$ With some abuse of notation, we note that any basis of $\B$ is automatically a $S$-approximate barycentric spanner of $\A$ with some positive number $S,$ \ie,
\begin{align}
S=\max_{j\in[d_0],\a\in\A}|\nu_{\a,j}|.
\end{align} 
However, since $\A$ does not span the $\Re^d$ space as required by Proposition \ref{bc_prop}, we cannot set $S~(>1)$ arbitrarily first with the hope that we can find the corresponding $S$-approximate barycentric spanner using Algorithm \ref{alg:barycentric} in Section \ref{sec:barycentric_app}.
\end{remark}
\subsection{OLS Estimator for General Networks}
With the new basis $\B$ at hand, we can almost follow what we have developed in Section \ref{ttc}, \ie, eq. (\ref{eq:ols}), to estimate $\bmu.$ But a more careful inspection suggests a crucial difference between identifiable network setting and the general network setting: since $\text{rank}\left(\B\right)=d_0<d,$ the $d$-by-$d$ matrix $\V_m=\left(\D_m^{\top}\D_m\right)=m\B\B^{\top}$ is singular, \ie, $\text{rank}(\V_m)=d_0<d$ for all $m\geq1.$ As a result, we cannot compute the OLS estimate of $\bmu$ the same as eq. (\ref{eq:ols}). 

To allow the \dm~to implement the \mttc~for general networks, we need to resolve the issues raised by the singularity of $\V_m.$ To this end, we use a slightly different version of OLS estimator \cite{RH17}, \ie, the OLS estimator of $\bmu$ after $m$ epochs of explorations is \begin{align}\hbmu_m=\left(\D_m^{\top}\D_m\right)^{\dagger}\D_m\r_m,\end{align}
where $\left(\D_m^{\top}\D_m\right)^{\dagger}$ denotes the Moore-Penrose pseudo-inverse of $\V_m=\left(\D_m^{\top}\D_m\right).$ We are now ready to state a new deviation inequality on the estimation errors. Here with some abuse of notations, we recall from inequality (\ref{eq:ub_bs1}) that $S$ is the upper bound on the absolute value of $\nu_{\a,i}$ for all $\a\in\A$ and $i\in[d_0].$  
\begin{theorem}
	\label{ad0}
	For a given positive integer $m,$ the probability that there exists a path $\a\in\A,$ such that the estimated mean delay of $\a$ deviates from its mean delay by at least $SR\sqrt{{32\ln (6)d_0^2+32d_0\ln{\delta^{-1}}}/{m}}$ is at most $\delta,$ after $m$ epochs of explorations, \ie,
	\begin{align*}
	\Pr\left(\left|\langle \a,\bmu\rangle-\langle \a,\hbmu\rangle\right|\geq SR\sqrt{\frac{32\ln (6)d_0^2+32d_0\ln{\delta^{-1}}}{m}}\right)\leq\delta.
	\end{align*}
\end{theorem}
\begin{proof}
	Please refer to Section \ref{sec:ad0} for the complete proof.
\end{proof}
\subsection{Upper Bounding $S$ and Obtaining Low Regrets}
By design of the \mttc, we only need to change the following parameters according to Theorem \ref{ad0}:
\begin{align}
&\text{The length of each epoch}=d_0,\\
\label{def:Delta}&\tilde{\Delta}_m=SR\sqrt{\frac{32\ln (6)d_0^2+96d_0\ln{T}}{m}}, \forall m=1,2,\ldots,\\
&\overline{n}=\sqrt{T}S^2R^2(32\ln (2)d_0+96\ln{T})/d^2_0,
\end{align}
and the \mttc~should achieve nearly optimal instance-dependent and worst case regrets.

However, to use the $\tilde{\Delta}_m$'s defined in (\ref{def:Delta}) as input parameters, we need to know the value of $S.$ Consider the matrix $\C_{d_0}=\left(\f_1,\ldots,\f_d\right),$ \ie, the matrix right after the termination of the while-loop in Algorithm \ref{general_network:alg}. By design of Algorithm \ref{general_network:alg}, we know that $\B$ is the first $d_0$ columns of $\C_{d_0},$ \ie, $\B=\C_{d_0}\left(:,[1:d_0]\right).$ By definition of $\bnu_a$, we can write
\begin{align}\a=\B\bnu_a=\sum_{i=1}^{d_0}\nu_{\a,i}\B(:,i)=\sum_{i=1}^{d_0}\nu_{\a,i}\C_{d_0}(:,i).\end{align}
We then make an observation: $\forall j\in[d_0],$ replacing $\C_{d_0}(:,j)$ by $\a:$
\begin{align}
\nonumber\left|\det\left(\a,\C_{d_0}(:,-j)\right)\right|=&\left|\det\left(\C_{d_0}(:,-j),\sum_{i=1}^{d_0}\nu_{\a,i}\C_{d_0}(:,i)\right)\right|\\
\label{eq:gn3}=&\left|\sum_{i=1}^{d_0}\nu_{\a,i}\det\left(\C_{d_0}(:,-j),\C_{d_0}(:,i)\right)\right|\\
\label{eq:gn4}=&\left|\nu_{\a,j}\det\left(\C_{d_0}(:,-j),\C_{d_0}(:,j)\right)\right|\\
\label{eq:gn5}=&\left|\nu_{\a,j}\right|\left|\det\C_{d_0}\right|,
\end{align}
where eq. (\ref{eq:gn3}) follows from the linearity of $\det(\cdot)$ operator and eq. (\ref{eq:gn4}) follows from the fact that the determinant of a matrix is 0 if a matrix has two identical columns. Re-arranging the terms in eq. (\ref{eq:gn5}), we have that 
\begin{align}
\left|\nu_{\a,j}\right|=\frac{\left|\det\left(\a,\C_{d_0}(:,-j)\right)\right|}{\left|\det\C_{d_0}\right|}
\end{align}
and by Remark \ref{remark:S}, we know that
\begin{align}
\label{opt:gn1}
S=\max_{j\in[d_0],\a\in\A}\left|\nu_{\a,j}\right|=\max_{j\in[d_0],\a\in\A}\frac{\left|\det\left(\a,\C_{d_0}(:,-j)\right)\right|}{\left|\det\C_{d_0}\right|}.
\end{align}
As demonstrated in Section \ref{sec:general_networks_basis}, the optimization problem at the RHS of (\ref{opt:gn1}) can be computed efficiently by first computing $$\max_{\a\in\A}\left|\det\left(\a,\C_{d_0}(:,-j)\right)\right|$$ for every $j\in[d_0]$ individually, and then taking maximum over $j\in[d_0].$

With all the above results, the regret bound of \mttc~for unidentifiable networks also follows immediately from that of Theorem \ref{mttcthm}.
\begin{theorem}
	\label{mttcthm_uniden}
The expected regret of \mttc~is bounded as
\begin{itemize}
	\item \textit{Instance-dependent regret:} 
	\begin{align*}
	O\left(\frac{\left(d_0^2\ln{T}+d_0^3\right)S^2\Delta_{\max}}{{\Delta}_{\min}^2}\right).
	\end{align*}
	\item \textit{Worst case regret:}
	\begin{align*}
	\widetilde{O}\left(d\sqrt{T}\right).
	\end{align*}
\end{itemize}
\end{theorem}
The proof of Theorem \ref{mttcthm_uniden} is omitted as it is very similar to that of Theorem \ref{mttcthm}.

\section{Numerical Experiments}
\label{sec:numerical}
In this section, we conduct extensive numerical experiments on synthetic data to validate the performances of the \ttc~and the \mttc~in terms of time average regret, \ie, regret/number of rounds, and computational efficiency.

We first present the setup of our numerical experiments. We vary $T$ from $5000$ to $25000$ with a step size of $5000.$ We set $\mu_{\max}=1000$ to allow enough heterogeneity in each link's delay distribution, and we use normal distribution with $R=0.1$ and $1$ for the noise terms. To demonstrate that our algorithms work well for networks with complicated topology and various scales, we use the \emph{grid network} structure. In a grid network, the underlay network is a $p\times p$ grid, and each node is able to route a packet to both the node on the right and beneath it (if exists). Fig. \ref{fig:grid_eg} shows the network with a $4$-by-$4$ underlay grid. In our experiment, we consider grid networks with $p=2,4,6,$ and $8.$
\begin{figure}[!ht]
	\subfigure{\label{fig:partite_eg}\includegraphics[width=7.5cm,height=4.5cm]{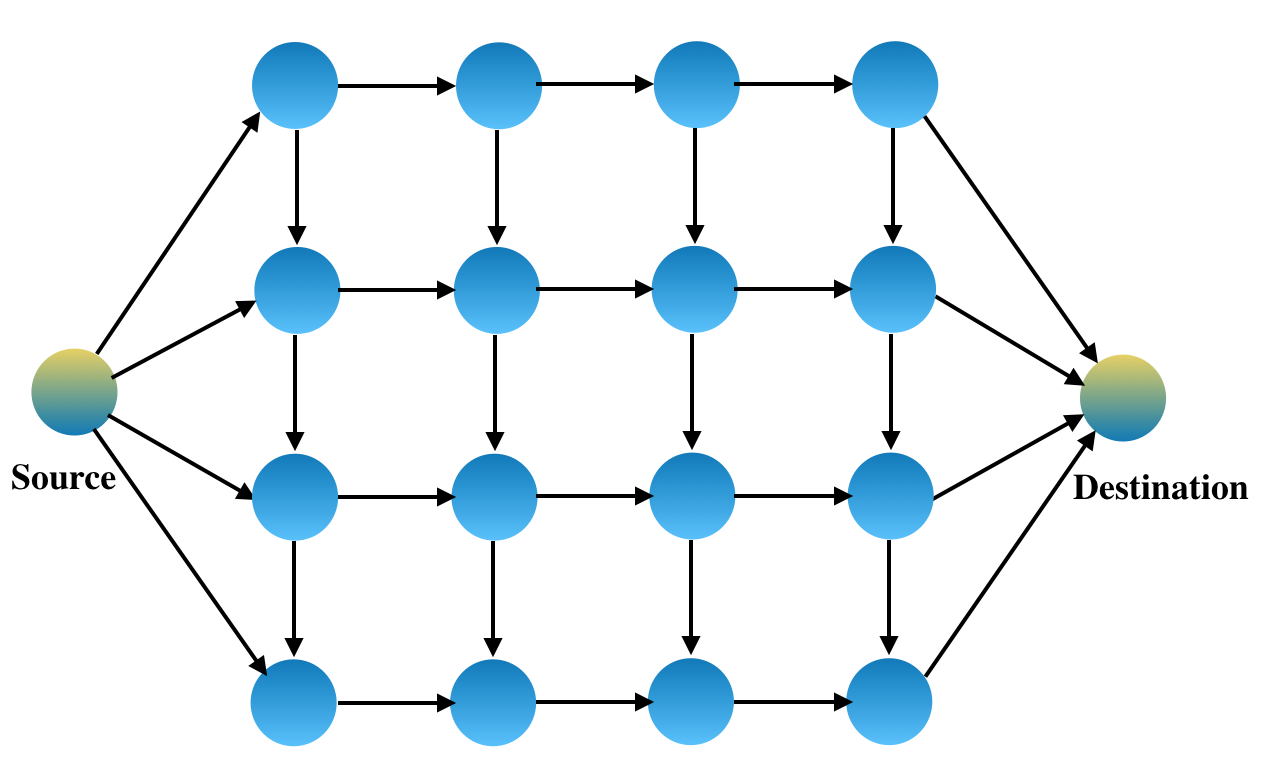}}
	\caption{Example of $4$-by$4$ grid network}	\label{fig:grid_eg}
\end{figure}
\begin{figure*}[!ht]
	\centerline{
		\hspace{-10mm}
		\subfigure[$p=2,R=0.1$]{\label{fig:grid2_01}\includegraphics[width=5cm,height=4cm]{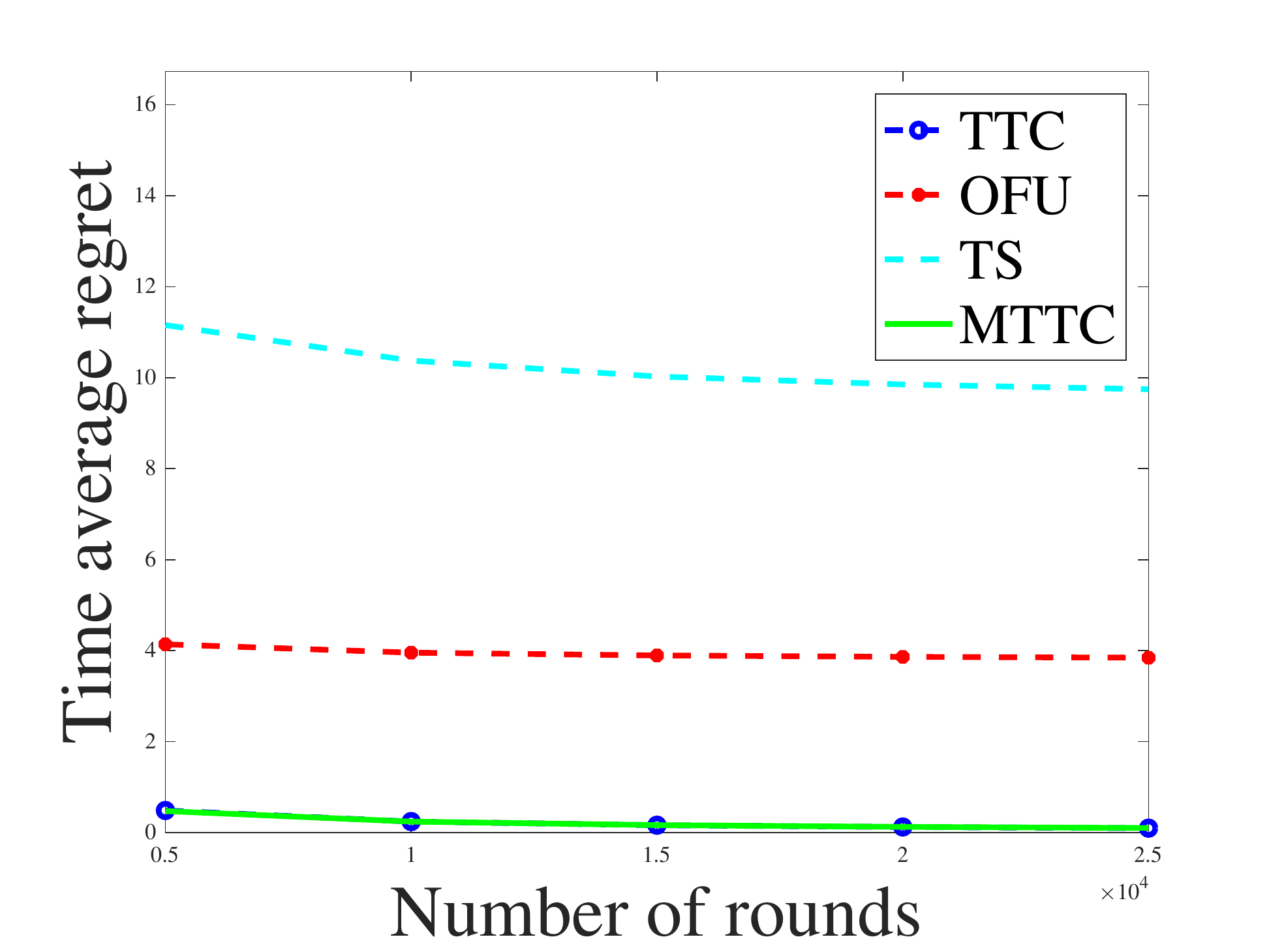}}\hspace{-5mm}
		\subfigure[$p=4,R=0.1$]{\label{fig:grid4_01}\includegraphics[width=5cm,height=4cm]{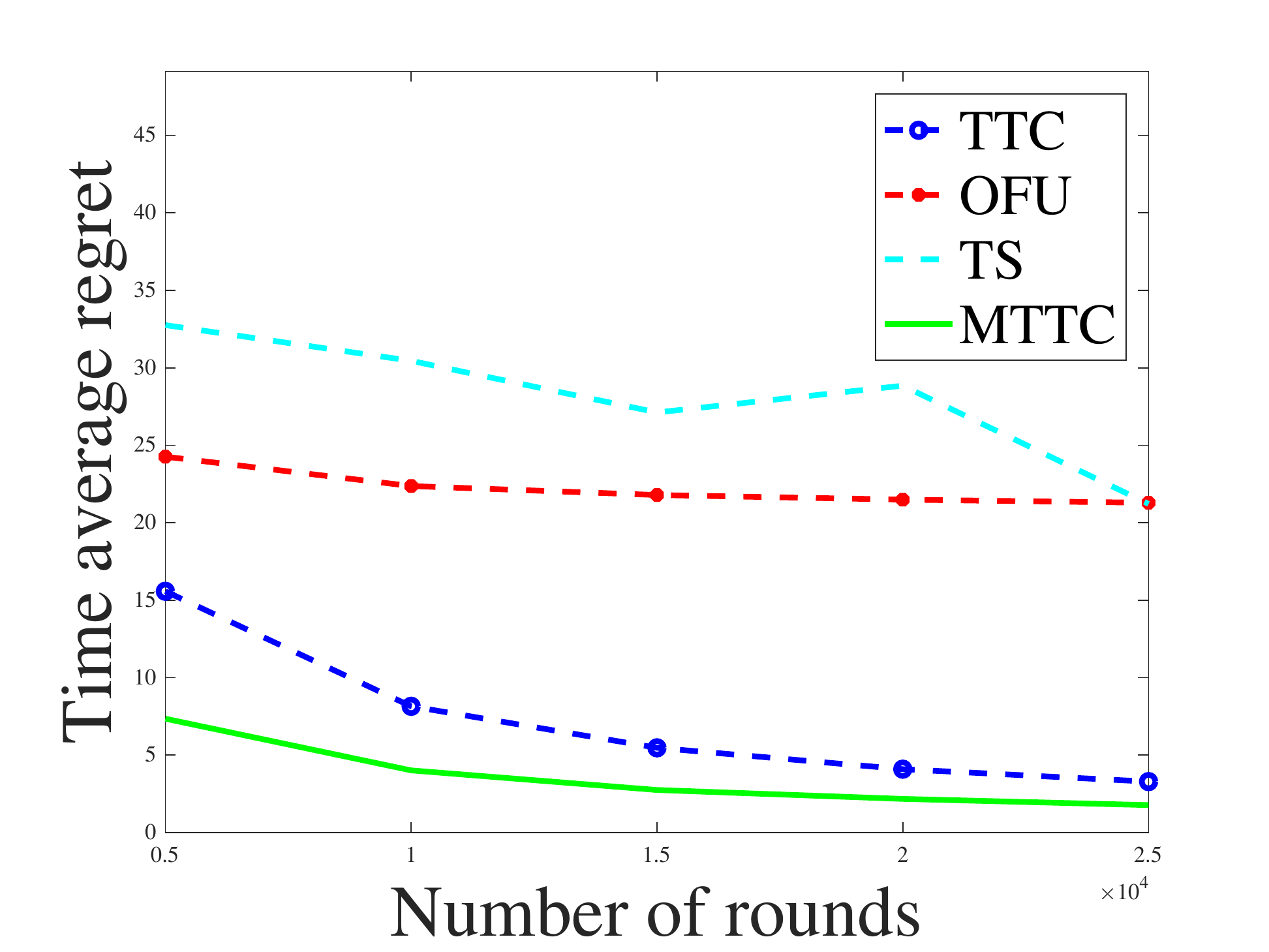}}\hspace{-5mm}
		\subfigure[$p=6,R=0.1$]{\label{fig:grid6_01}\includegraphics[width=5cm,height=4cm]{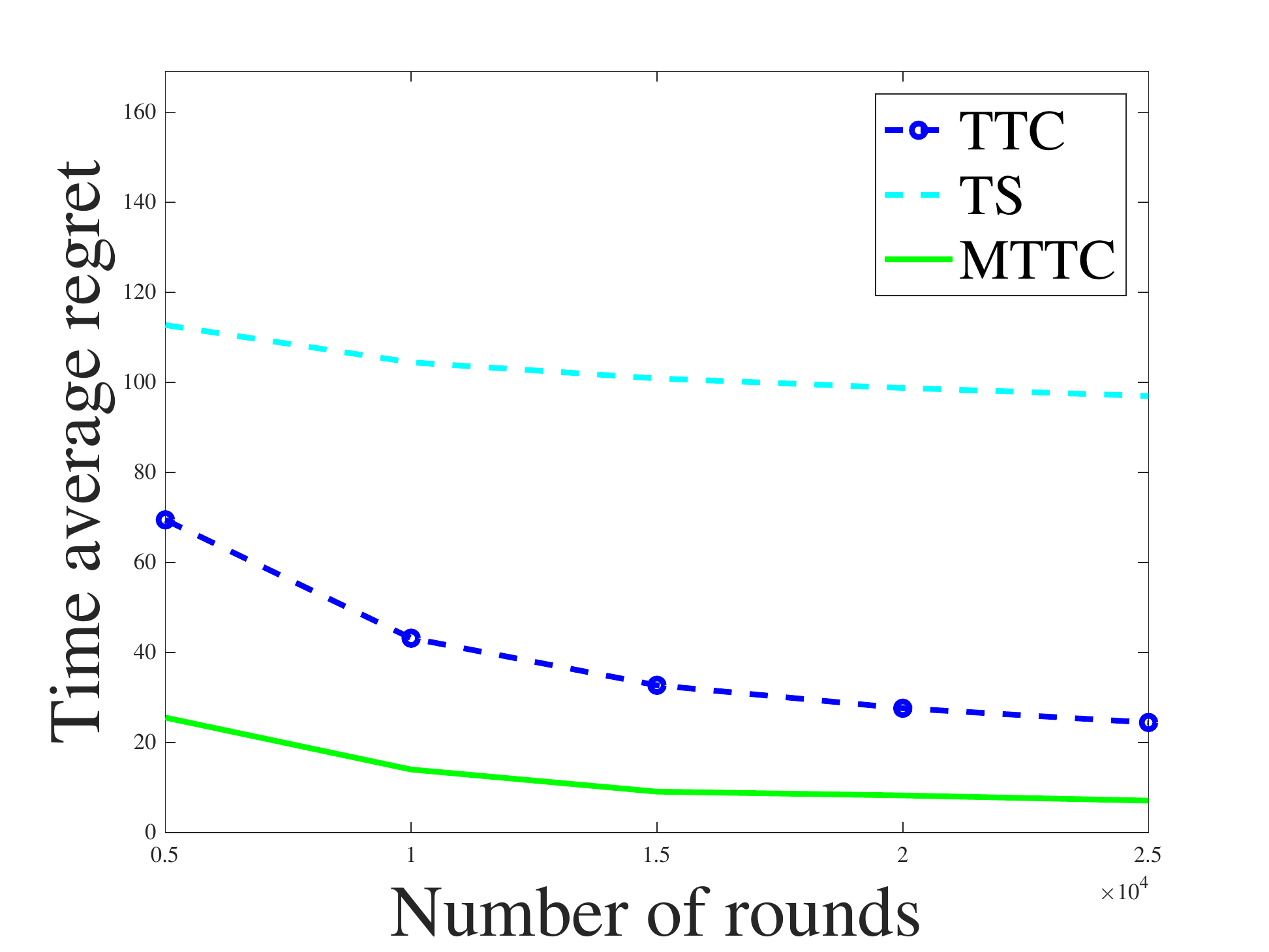}}\hspace{-5mm}
		\subfigure[$p=8,R=0.1$]{\label{fig:grid8_01}\includegraphics[width=5cm,height=4cm]{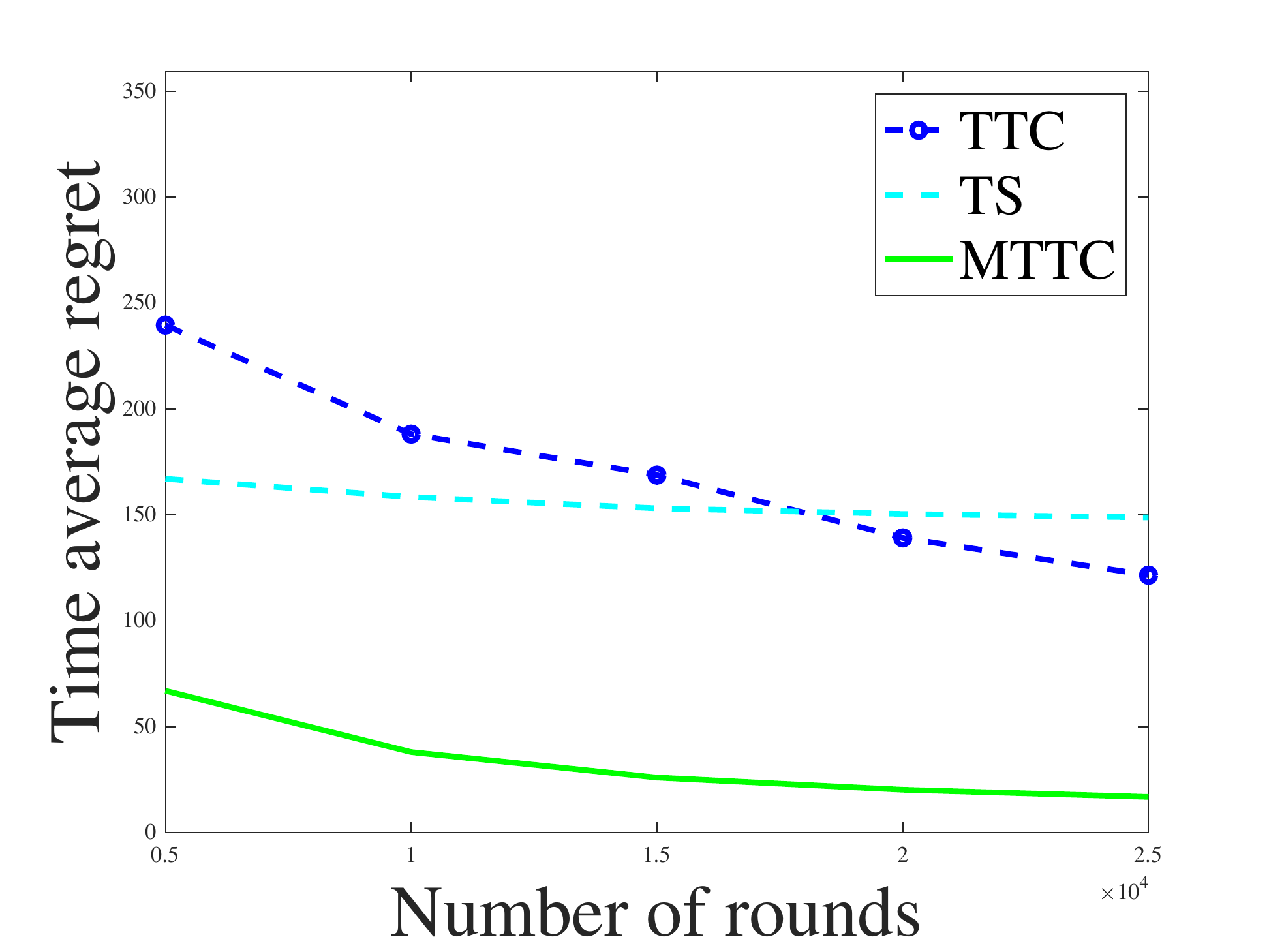}}\hspace{-10mm}}
	\vspace{-3mm}
	\centerline{
		\hspace{-10mm}
		\subfigure[$p=2,R=1$]{\label{fig:grid2_1}\includegraphics[width=5cm,height=4cm]{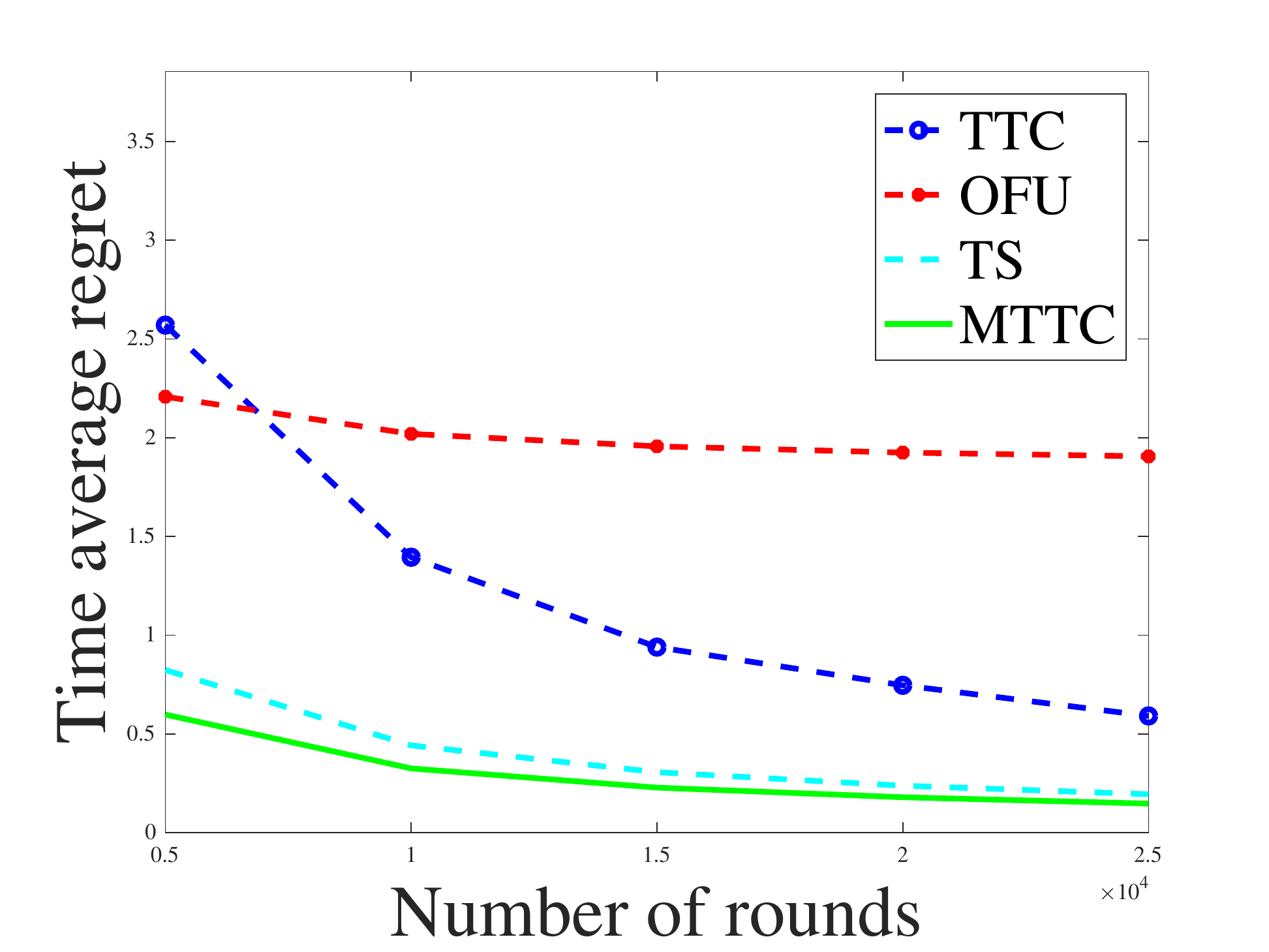}}\hspace{-5mm}
		\subfigure[$p=4,R=1$]{\label{fig:grid4_1}\includegraphics[width=5cm,height=4cm]{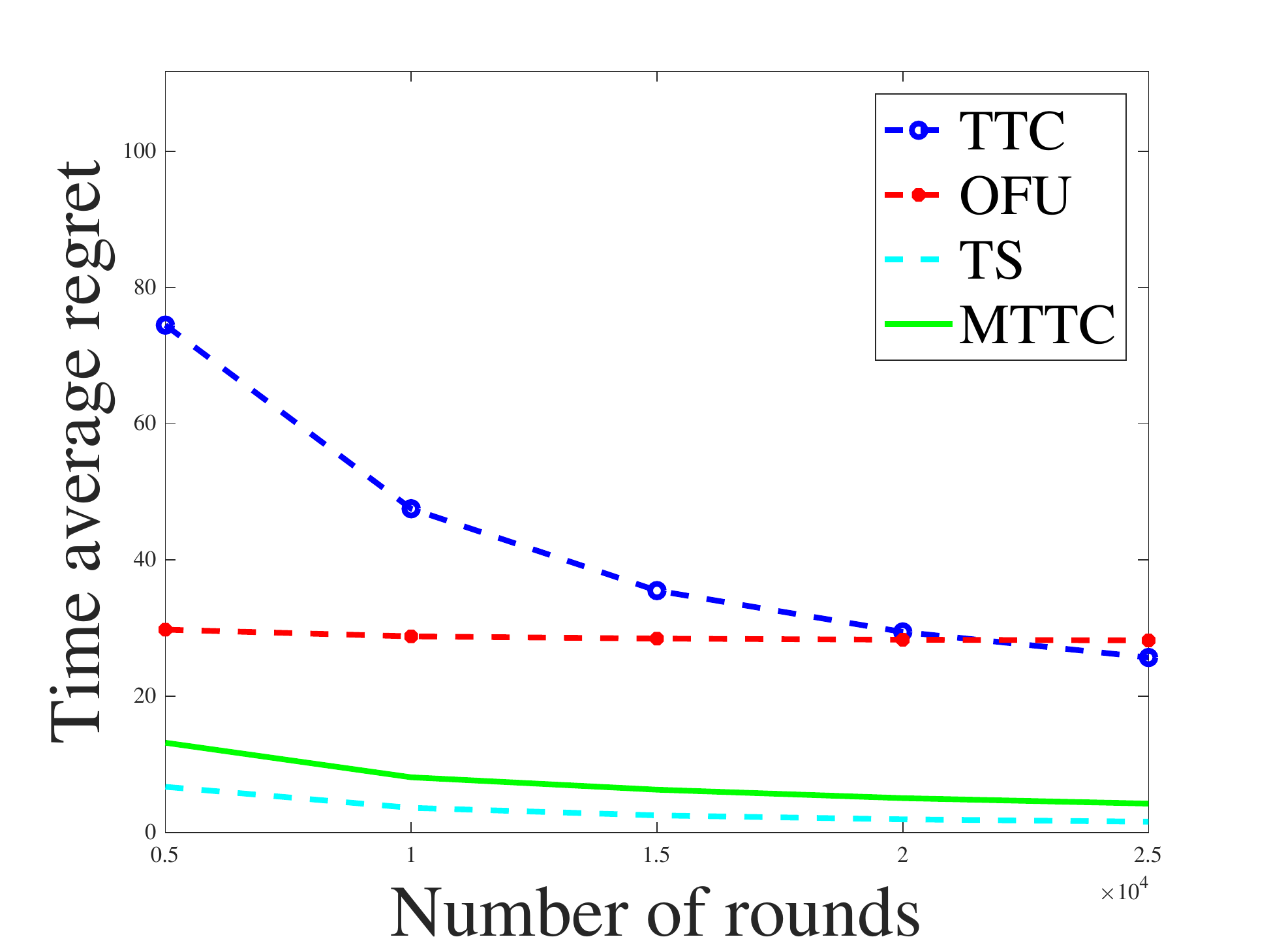}}\hspace{-5mm}
		\subfigure[$p=6,R=1$]{\label{fig:grid6_1}\includegraphics[width=5cm,height=4cm]{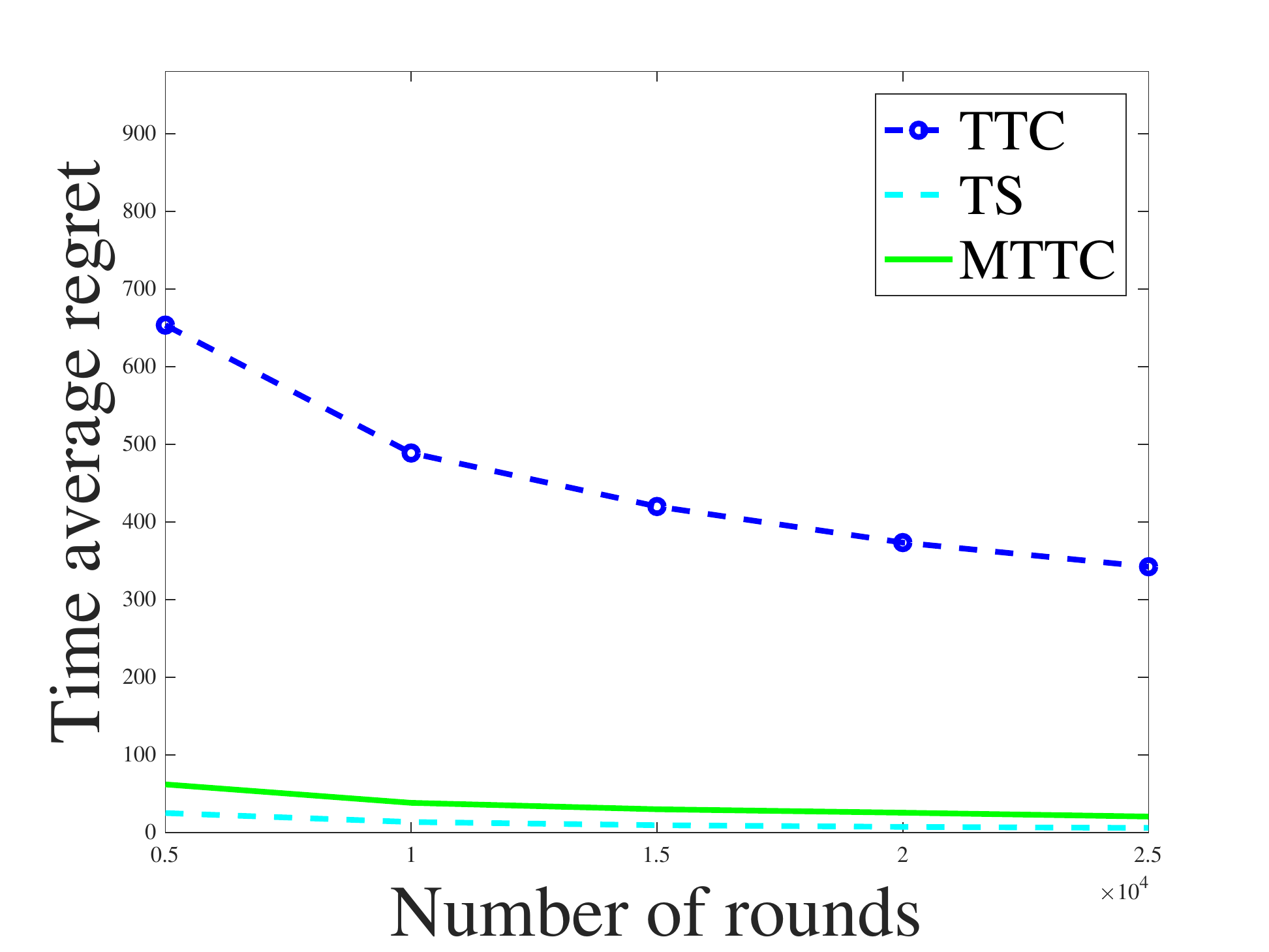}}\hspace{-5mm}
		\subfigure[$p=8,R=1$]{\label{fig:grid8_1}\includegraphics[width=5cm,height=4cm]{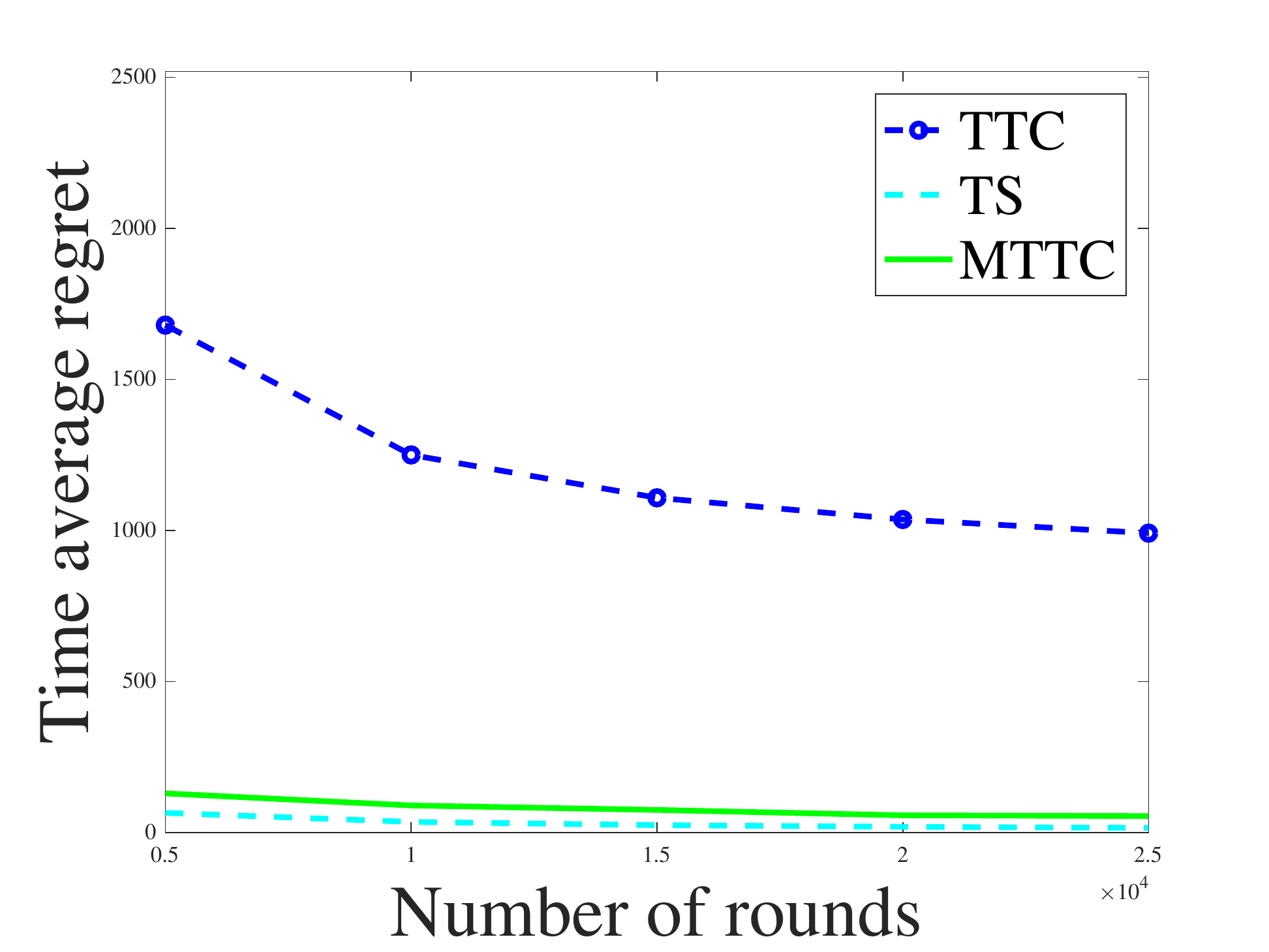}}\hspace{-10mm}}
	\vspace{-4mm}
	\caption{Plots of results}
	\label{results_grid}
\end{figure*}
For ease of implementation, we use the Thompson sampling (TS) algorithm proposed in \cite{AG13} as the alternative in \mttc, and compare the performance of our algorithms with the (inefficient) OFU algorithm proposed in \cite{AYPS11} and the Thompson sampling algorithm \cite{AG13}. Also for fair comparisons, we set a hard computation budget of 90 seconds for each instance, and call a runtime error once the runtime of an algorithm exceed this limit. Finally, all the results presented here are averaged over 200 iterations.

We first describe some basic statistics as well as the runtime of different algorithms in Table \ref{table:grid} to visualize the scales and complexities of the networks for different values of $p.$ As we can see, the OFU and the TS algorithms consume tens to hundreds of times more runtime than the \ttc~for all the cases. When $k\geq 6,$ the runtime of the OFU algorithm exceeds the computation budget.
\begin{table}[!ht]
	\begin{center}
		\begin{tabular}{ |c|c|c|c|c|}
			\hline
			$p=$&2&4&6&8\\
			\hline
			$d:$ \#links&8&32&72&128\\
			\hline
			$d_0:$ size of basis&4&16&36&64\\
			\hline
			$|\A|:$ \#paths&4&56&792&11440\\
			\hline
			minimum \#hop&3&5&7&9\\
			\hline
			maximum \#hop&4&8&12&16\\
			\hline
			runtime of \ttc~(s)&0.01&0.16&1.00&1.14\\
			\hline
			runtime of \mttc~(s)&0.10&5.78&18.65&35.08\\
			\hline
			runtime of OFU algorithm (s)&1.55&9.43&>90&>90\\
			\hline
			runtime of TS algorithm (s)&29.60&35.14&39.81&51.96\\
			\hline
		\end{tabular}
		\caption{Basic statistics and average runtime of different algorithms for grid networks when $R=1$}
		\label{table:grid}
	\end{center}
\end{table}

The results of time average regret are shown in Fig. \ref{results_grid}. From the plots, we can read that the time average regrets of the \ttc~and the \mttc~are significantly lower than those of the OFU and the TS algorithms' in the low noise (or $R=0.1$) case, especially when the number of rounds is large. The only exception is when $p=8$ and $T\leq 15000,$ the time average regret of the \ttc~is larger than that of the TS algorithm. This is because as $p$ increases, the value of $S$ also increases, and it thus takes longer time for the \ttc~to identify the optimal path. For the high noise case, we can see that the performances of the \mttc~and the TS algorithm are close. Although the performances of the \ttc~is worse than the TS algorithm, we believe the \ttc~and the \mttc~can outperform the TS algorithm as the $T$ increases.
\begin{figure}[!ht]
	\hspace{-15mm}
	\subfigure[$p=2,R=1$]{\label{fig:grid2_01}\includegraphics[width=5cm,height=4cm]{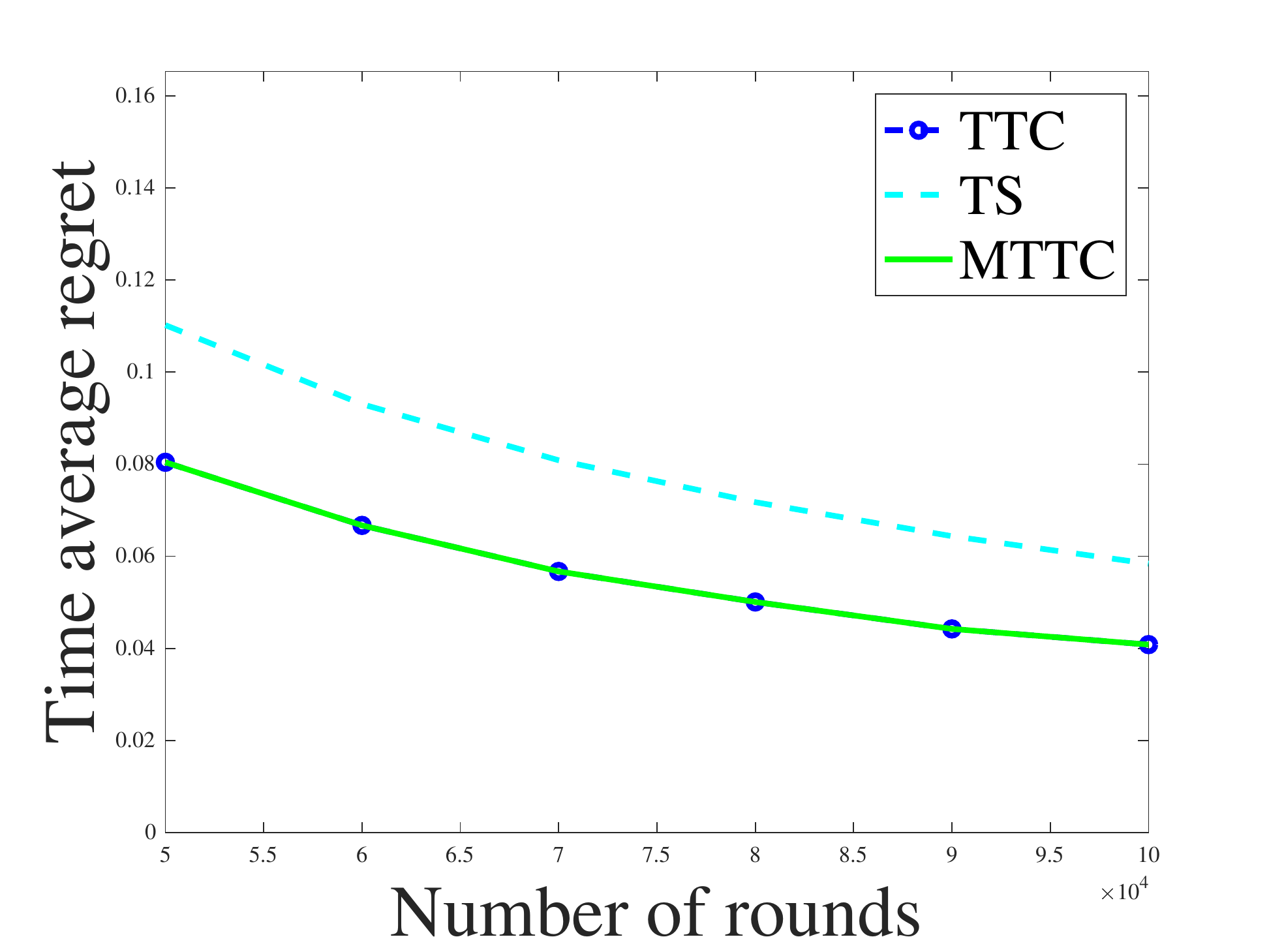}}\hspace{-5mm}
	\subfigure[$p=4,R=1$]{\label{fig:grid4_01}\includegraphics[width=5cm,height=4cm]{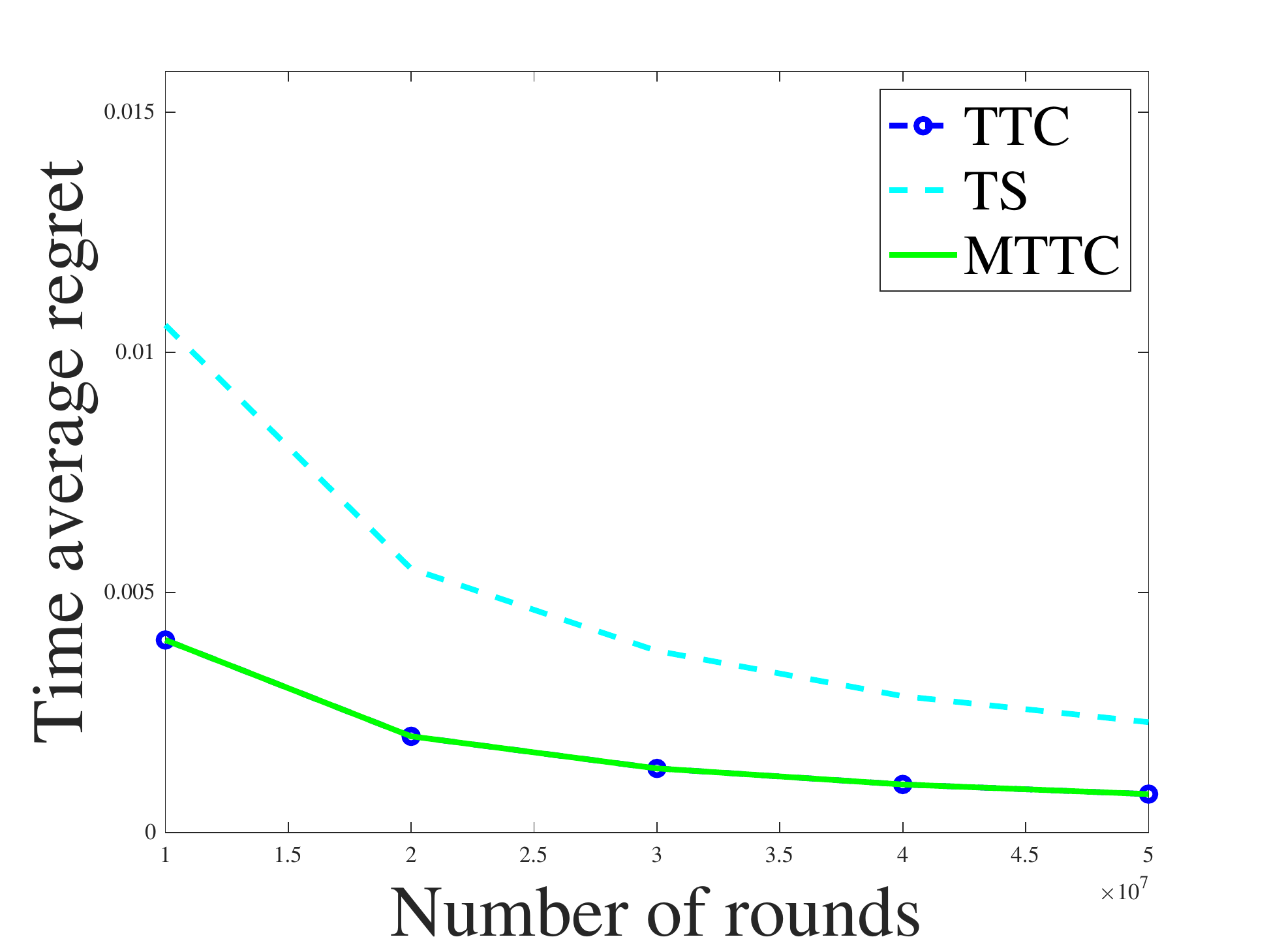}}\hspace{-15mm}
	\vspace{-4mm}
	\caption{Additional results for grid networks}
	\label{results_additional}
\end{figure}

To verify our conjecture, we conduct additional experiment for the $R=1$ case with larger $T$ with $p=2$ and $p=4$. When $p=2,$ we vary $T$ from $5\times 10^4$ to $10^5$ with a step size of $10^4;$ when $p=4,$ we vary $T$ from $10^8$ to $5\times10^8$ with a step of $10^8.$ The plots in Fig. \ref{results_additional} clearly show that when $T$ becomes large enough, the \ttc~and the \mttc~finish with lower regret than the TS algorithm. 
\section{Related Works}
\label{sec:related_works}
Stochastic multi-armed bandits is a prevalent framework for sequential decision-making. Early work on stochastic MAB problems \cite{Robbins52,LaiRobbins85,gittins2011multi} tended to be more focused on asymptotic guarantees, whereas more recent work \cite{ABF02,auer2003using} has been directed towards a non-asymptotic analysis in which regret can be bounded over a fixed time horizon $T$. Two of the best-known and well-studied techniques are known as the UCB algorithm that follows the OFU principle \cite{ABF02} and the explore then exploit algorithm \cite{AO10,S17}. Recently, the Bayesian setting accompanied by the \emph{Thompson Sampling} (TS) technique has also been thoroughly analyzed due to the ease of implementation and favorable empirical results \cite{RVR14}.

To model inter-dependence relationships among different arms, models for stochastic linear bandits have also been studied. In stochastic linear bandits, each action can be described by a finite number of features, and the expected reward/loss function is linear in these features. The reward/loss function can thus be expressed as a vector in $\Re^d$, and the uncertainty arises from the noisy feedbacks of the observed rewards/losses. In \cite{A02,CLRS11}, the authors consider stochastic linear bandits with fixed finite action sets; while in \cite{DHK08,AYAS09,RT10}, stochastic linear bandits with possibly infinite cardinality of actions has been studied. The authors of \cite{AYPS11} unify these two lines of research, and have proposed the state-of-art algorithm for the problem. All these algorithms follow essentially the OFU principle. But the OFU-inspired algorithms are impractical to run when the number of actions become large as they all require the solution of a NP-hard bilinear optimization problem. TS algorithms proposed in \cite{RVR14,AG13,AL17} are able to bypass the high computational complexities provided that the \dm~can efficiently sample from the posterior on the reward function. Unfortunately, achieving optimal regret bound via TS algorithms is possible only if the true prior over the reward/loss vector is known. To further capture the non-stochastic aspect of linear bandits, adversarial linear bandits in which the reward/loss vector can change over time arbitrarily (or even adversarially) have been studied \cite{AK04,AHR09,BE15}. Among them, \cite{BE15} gave an efficient strategy with optimal regret when the action set is convex, and one can do efficient linear optimization on the action set. Since adversarial setting is not the main topic of this paper,  interested readers can refer to \cite{BE15} and the references therein.

A special case of linear bandits is combinatorial bandits where the action set is constrained to subset of $\{0,1\}^d.$ In combinatorial stochastic bandits, it is often assumed that the reward/loss vector is observed at all the coordinates sampled by the action taken. This is the so-called semi-bandit feedback setting \cite{ABL11}. The authors of \cite{GKJ12} initiated the study of combinatorial stochastic bandits under semi-bandit feedback and a network-structured action set; while \cite{CWY13} studied the general action set case. The authors of \cite{KWAS15} further characterized tight upper and lower bounds for this problem. Assuming the noise is independent across different coordinates, the authors of \cite{TZCPJ18} improved upon the results obtained in \cite{KWAS15}. For the bandit feedback case, the authors of \cite{LZ12} give algorithms that require brute-force search over the action space with instance dependent regret $\widetilde{O}(d_0^3d\ln T/\Delta_{{\min}})$. For adversarial combinatorial bandits, the authors of \cite{ABL12} presented the efficient and optimal algorithm for the semi-bandit feedback case while the authors of \cite{HKM16} described an optimal algorithm for the bandit feedback case, but its computational complexity scales linearly with the number of actions.
\section{Conclusion}
\label{sec:conclusion}
In this paper, we developed efficient algorithms with nearly optimal regrets for the problem of stochastic online shortest path routing with end-to-end feedback. Starting with the identifiable networks, we introduced the \etc~to obtain nearly optimal instance-dependent and sub-optimal worst case regrets efficiently. We then presented the adaptive \ttc~and the \mttc~to achieve nearly optimal instance-dependent and worst case regrets. Afterwards, we extended our results to general networks. Finally, we conducted extensive numerical experiments to demonstrate the superior regret performances and computational efficiency of our proposed algorithms.
\bibliographystyle{ACM-Reference-Format}
\bibliography{IEEEabrv,paperlist,reflist-short}
\newpage
\appendix
\section{Proofs}
\label{op}
\subsection{Proof of Theorem \ref{ad}}
\label{sec:ad}
\begin{proof}
	The proof of Theorem \ref{ad} relies on a deviation inequality of the OLS estimator, which we state as a lemma here.
	\begin{lemma}
		\label{seqDesign}
		For a given positive integer $m,$ the probability that the difference between $\hbmu_m$ and $\bmu$ under the $\V_m$ norm is not less than $R\sqrt{{2d+3\ln{\delta^{-1}}}}$ is at most $\delta,$ after $m$ epochs of explorations, \ie, 
		\begin{align} \textstyle
		\nonumber\Pr\left(\|\hbmu_m-\bm{\mu}\|_{\V_m}\geq R\sqrt{2\ln (2)d+4\ln{\delta^{-1}}}\right)\leq\delta,
		\end{align}
		where $\V_m=\D_m^{\top}\D_m.$
	\end{lemma}
	\begin{proof}{Proof of Lemma~\ref{seqDesign}}
		The proof of Lemma~\ref{seqDesign} uses the Laplace's method, and this is mostly adopted from \cite{S16}. Following the proof in \cite{S16}, denote $\w=\r_m-\D\bmu,$ as the vector of noise, and $\z=\D_m^{\top}\w,$ we have 
		\begin{align}
		\nonumber\hbmu_m-\bmu=&\left(D_m^{\top}\D_m\right)^{-1}\D_m^{\top}\r_m-\bmu\\
		\nonumber=&\left(D_m^{\top}\D_m\right)^{-1}\D_m^{\top}\left(\w+\D_m\bmu\right)-\bmu\\
		\nonumber=&\V_m^{-1}\D_m^{\top}\w\\
		=&\V_m^{-1}\z.
		\end{align}
		Therefore, $\|\hbmu_m-\bmu\|_{\V_m}^2=\|\z\|_{\V_m^{-1}}^2.$ In order to get a tail bound on the quantity $\|\hbmu-\bmu\|_{\V_m},$ we can instead work on $\|\z\|_{\V_m^{-1}}.$ Consider $\E\left[\exp\left(\bm\alpha^{\top}\z\right)\right],$ the moment generating function of $\z$ with respect to $\bm{\alpha}\in\Re^d,$ by the $R$-sub-Gaussian property of $\eta_1,\ldots,\eta_{nd},$ 
		\begin{align}
		\nonumber&\E\left[\exp\left(\bm{\alpha}^{\top}\z_t\right)\right]\\
		\nonumber=&\E\left[\exp\left(\sum_{s=1}^{md}\left(\bm{\alpha}^{\top}\a_{I_s}\right)\eta_s\right)\right]\\
		\nonumber=&\E\left[\E_{\eta_{md}}\left[\exp\left(\sum_{s=1}^{md}\left(\bm{\alpha}^{\top}\a_{I_s}\right)\eta_s\right)\middle|\a_{I_1},\ldots,\a_{I_{md}},\eta_1,\ldots,\eta_{md-1} \right]\right]\\
		\nonumber\leq&\E\left[\exp\left(\sum_{s=1}^{md-1}\left(\bm{\alpha}^{\top}\a_{I_s}\right)\eta_s\right)\right]\exp\left(\frac{\left(\bm\alpha^{\top}\a_{I_{md}}R\right)^2}{2}\right)\\
		\nonumber&\vdots\\
		\nonumber\leq&\exp\left(\sum_{s=1}^{md}\frac{\left(\bm\alpha^{\top}\a_{I_s}R\right)^2}{2}\right)\\
		=&\exp\left(\frac{R^2}{2}\bm\alpha^{\top}\V_m\bm{\alpha}\right).
		\end{align}
		One can rewrite this to $$\E\left[\exp\left(\bm{\alpha}^{\top}\z-\frac{R^2}{2}\bm\alpha^{\top}\V_m\bm{\alpha}\right)\right]\leq1.$$ We further define $$M_{\bm{\alpha}}=\exp\left(\bm{\alpha}^{\top}\z-\frac{R^2}{2}\bm\alpha^{\top}\V_m\bm{\alpha}\right)$$ and $$\overline{M}=\int M_{\bm{\alpha}}h(\bm{\alpha})d\bm{\alpha},$$ where $h(\cdot)$ is the density of the $\mathcal{N}\left(0,\V_m^{-1}/R^2\right).$ Now we have
		\begin{align}
		\nonumber\overline{M}=&\frac{R^d}{\sqrt{(2\pi)^d\det \V_m^{-1}}}\int\exp\left(\bm{\alpha}^{\top}\z-{R^2}\bm\alpha^{\top}\V_m\bm{\alpha}\right)d\bm{\alpha}\\
		\nonumber=&\frac{R^d}{\sqrt{(2\pi)^d\det \V_m^{-1}}}\int\exp\left(\frac{\z\V_m^{-1}\z}{4R^2}-R^2\left\|\bm{\alpha}-\frac{\V_m^{-1}\z}{2R^2}\right\|^2_{\V}\right)d\bm\alpha\\
		=&\frac{1}{2^{d/2}}\exp\left(\frac{\left\|\z\right\|^2_{\V_m^{-1}}}{4R^2}\right).
		\end{align}
		Note that $$\E[\overline{M}]=\int \E[M_{\bm{\alpha}}]h(\bm{\alpha})d\bm{\alpha}\leq1,$$ we have 
		$$\E\left[\exp\left(\frac{\left\|\z\right\|^2_{\V_m^{-1}}}{4R^2}\right)\right]\leq2^{{d}/{2}},$$
		and thus
		\begin{align}
		\Pr\left(\left\|\z\right\|^2_{\V_m^{-1}}\geq R^2\left(2\ln(2)d+4\ln\delta^{-1}\right)\right)\leq\delta
		\end{align}
		by Chernoff Bound. Note that $\|\hbmu_m-\bmu\|_{\V_m}^2=\|\z\|_{\V_m^{-1}}^2,$ we conclude the proof.  
	\end{proof}
	We are now ready to proof Theorem \ref{ad}. From Lemma \ref{seqDesign}, we have the probability that the difference between $\hbmu_m$ and $\bmu$ under the $\V_m$ norm is not less than $R\sqrt{2\ln(2)d+4\ln\delta^{-1}}$ is at most $\delta,$ \ie, 
	\begin{align} \textstyle
	\nonumber\Pr\left(\|\hbmu_m-\bmu\|_{\V_m}\geq R\sqrt{2\ln(2)d+4\ln\delta^{-1}}\right)\leq\delta.
	\end{align}
	Equivalently, we have with probability at least $1-\delta.$
	\begin{align}
	\label{ad1}
	\left(\hbmu_m-\bmu\right)^{\top}\V_m\left(\hbmu_m-\bmu\right)\leq\gamma^2,
	\end{align}
	where $$\gamma=R\sqrt{2\ln(2)d+4\ln\delta^{-1}}.$$ By definition of $\B$ and $\V_m,$ we know that $\V_m=m\B\B^{\top}.$ Denoting $\x=\B^{\top}\left(\hbmu_m-\bmu\right),$ (\ref{ad1}) indicates that with probability at least $1-\delta,$
	\begin{align}
	\|\x\|^2\leq\frac{\gamma^2}{m}.
	\end{align}
	As $\B$ is the $S$-approximate barycentric spanner of $\A,$ for any $\a\in\A,$ we have,
	\begin{align}
	\nonumber\forall\a\in\A\quad\langle\a,\hbmu_m-\bmu\rangle^2=&\left(\hbmu_m-\bmu\right)^{\top}\B\bnu_{\a}\bnu_{\a}^{\top}\B^{\top}\left(\hbmu_m-\bmu\right)\\
	\nonumber=&\x^{\top}\bnu_{\a}\bnu_{\a}^{\top}\x\\
	\nonumber=&\left(\x^{\top}\bnu_{\a}\right)^2\\
	\nonumber\leq&\|\x\|^2\|\bnu_{\a}\|^2\\
	\nonumber\leq&\frac{dS^2\gamma^2}{m}
	\end{align}
	holds with probability at least $1-\delta.$ Here the second last inequality follows from Cauchy-Schwarz inequality while the last inequality follows from inequality (\ref{eq:ub_bs}). Equivalently, we have
	\begin{align*}
	\Pr\left(\exists\a\in\A:\left|\langle \a,\bmu\rangle-\langle \a,\hbmu_m\rangle\right|\geq SR\sqrt{\frac{2\ln (2)d^2+4d\ln{\delta^{-1}}}{m}}\right)\leq\delta.
	\end{align*}
	This concludes the proof.
	\end{proof}
\subsection{Proof of Theorem \ref{etcthm}}
\label{sec:etcthm}
\begin{proof}
	We prove the instance dependent and worst case regrets separately.
	
	\noindent\textbf{1. Obtaining Instance Dependent Regret:}
	To keep track of the regret incurred by \etc~in the committing stage, we consider the event $\{\a_{I_{t_0}}\neq\a_*\}$ with $t_0=n\cdot d+1.$ This means that \etc~either has a large overestimate on $\langle \a_*,\bm{\mu}\rangle$, that is 
	\begin{align} \langle \a_*,\hbmu_n\rangle\geq\langle \a_*,\bm{\mu}\rangle+\frac{\Delta_{I_{t_0}}}{2},\end{align}
	or a major underestimate on $\langle \a_{I_{t_0}},\bm{\mu}\rangle$, that is
	\begin{align} \langle \a_{I_{t_0}},\hbmu_n\rangle\leq \langle \a_{I_{t_0}},\bm{\mu}\rangle-\frac{\Delta_{I_{t_0}}}{2}.\end{align}
	By union bound, we have
	\begin{align}
	\nonumber\Pr\left(\left\{\a_{I_{t_0}}\neq\a_*\right\}\right)\leq&\Pr\left(\langle \a_*,\hbmu_n\rangle\geq\langle \a_*,\bm{\mu}\rangle+\frac{\Delta_{I_{t_0}}}{2}\right)\\
	\label{eq:etcthm1}&+\Pr\left(\langle \a_{I_{t_0}},\hbmu_n\rangle\leq \langle \a_{I_{t_0}},\bm{\mu}\rangle-\frac{\Delta_{I_{t_0}}}{2}\right).
	\end{align}
	We then work on the two quantities separately. By definition of $\Delta_{{\min}},$ we have $\Delta_{{\min}}\leq\Delta_{I_{t_0}},$ and this further indicates
	\begin{align}
	\nonumber\Pr\left(\langle \a_*,\hbmu_n\rangle\geq\langle \a_*,\bm{\mu}\rangle+\frac{\Delta_{I_{t_0}}}{2}\right)\leq&\Pr\left(\langle \a_*,\hbmu_n\rangle\geq\langle \a_*,\bm{\mu}\rangle+\frac{\Delta_{{\min}}}{2}\right)\\
	\label{eq:etcthm}\leq&\Pr\left(\langle \a_*,\hbmu_n-\bmu\rangle^2\geq\frac{\Delta^2_{I_{\min}}}{4}\right).
	\end{align}
	To this end, we apply the results from Theorem \ref{ad}. Specifically, we set 
	\begin{align}\delta=\exp\left(\frac{\ln(2)d}{2}-\frac{\Delta^2_{\min}n}{16dS^2R^2}\right)\end{align} so that
	\begin{align}\frac{\Delta_{\min}}{2}=SR\sqrt{\frac{2\ln (2)d^2+4d\ln{\delta^{-1}}}{n}}.\end{align}
	Therefore, we can further upper bound inequality (\ref{eq:etcthm}) as
	\begin{align*}
	 &\Pr\left(\langle \a_*,\hbmu_n\rangle\geq\langle \a_*,\bm{\mu}\rangle+\frac{\Delta_{I_{t_0}}}{2}\right)\\
	\leq&\Pr\left(\langle \a_*,\hbmu_n-\bmu\rangle^2\geq SR\sqrt{\frac{2\ln (2)d^2+4d\ln{\delta^{-1}}}{n}}\right)\\
	\leq&\delta.
	\end{align*}	
	Similarly,
	\begin{align*}\Pr\left(\langle \a_{I_{t_0}},\hbmu_n\rangle\leq \langle \a_{I_{t_0}},\bm{\mu}\rangle-\frac{\Delta_{I_{t_0}}}{2}\right)\leq\delta.
	\end{align*}
	Together with inequality (\ref{eq:etcthm1}), we have 
	\begin{align}
	\Pr\left(\left\{\a_{I_{t_0}}\neq\a_*\right\}\right)\leq2\delta.
	\end{align} 
	 the expected regret of \etc~is upper bounded as
	\begin{align}
	\nonumber&\E\left[\textnormal{Regret(\etc)}\right]\\
	\nonumber\leq&n\cdot d\Delta_{\max}+2(T-n\cdot d)\Delta_{\max}\delta\\
	\leq &n\cdot d\Delta_{\max}+2T\Delta_{\max}\exp\left(\frac{\ln(2)d}{2}-\frac{\Delta^2_{\min}n}{16dS^2R^2}\right).
	\end{align}
	Setting \begin{align}n=\frac{16dS^2R^2\ln(dT)+8\ln(2)d^2S^2R^2}{\Delta_{\min}^2}\end{align} brings us 
	\begin{align}
	\nonumber&\E\left[\text{Regret}_T(\text{\etc})\right]\\
	\leq&\frac{\left(16d^2S^2R^2\ln \left(dT\right)+8\ln(2)d^3S^2R^2\right)\Delta_{\max}}{\Delta_{\min}^2}+2.
	\end{align}
	
	\noindent\textbf{2. Obtaining Worst Case Regret:}
	
	We again consider the event $\{\a_{I_{t_0}}\neq\a_*\}$ with $t_0=n\cdot d+1.$ Note that this implies 
	\begin{align}
	 \a_{I_{t_0}},\hbmu\rangle\leq\langle \a_*,\hbmu\rangle.
	\end{align}
	From Theorem \ref{ad}, with probability at least $1-\delta,$
	\begin{align}
	&\langle \a_{I_{t_0}},\bmu\rangle\leq\langle \a_{I_{t_0}},\hbmu_n\rangle+SR\sqrt{\frac{2\ln (2)d^2+4d\ln{\delta^{-1}}}{n}}\\
	&\langle \a_*,\hbmu_n\rangle\leq\langle \a_*,\bmu\rangle+SR\sqrt{\frac{2\ln (2)d^2+4d\ln{\delta^{-1}}}{n}}.
	\end{align}
	Combining the above three inequalities, we have
	\begin{align}
	\nonumber\langle \a_{I_{t_0}},\bmu\rangle\leq&\langle \a_{I_{t_0}},\hbmu_n\rangle+SR\sqrt{\frac{2\ln (2)d^2+4d\ln{\delta^{-1}}}{n}}\\
	\nonumber\leq&\langle \a_*,\hbmu_n\rangle+SR\sqrt{\frac{2\ln (2)d^2+4d\ln{\delta^{-1}}}{n}}\\
	\label{eq:etcthm2}\leq&\langle \a_*,\bmu\rangle+2SR\sqrt{\frac{2\ln (2)d^2+4d\ln{\delta^{-1}}}{n}}.
	\end{align}
	Denoting the event $E_n$ as for all $\a\in\A,$ the absolute difference between $\langle\a,\hbmu_n\rangle$ and $\langle\a,\bmu\rangle$ is not larger than $SR\sqrt{{2\ln (2)d^2+4d\ln{\delta^{-1}}}/{n}},$ \ie,
	\begin{align*}
	E_n=\left\{\forall\a\in\A: \left|\langle\a,\bmu\rangle-\langle\a,\hbmu_n\rangle\right|\leq SR\sqrt{{2\ln (2)d^2+4d\ln{\delta^{-1}}}/{n}}\right\}
	\end{align*}
	By Theorem \ref{ad}, we know that
	\begin{align}
	\Pr\left(E_n\right)\geq1-\delta.
	\end{align}
	Therefore, the expected regret of \etc~is upper bounded as
	\begin{align}
	\nonumber&\E\left[\textnormal{Regret(\etc)}\right]\\
	\nonumber=&\E\left[\textnormal{Regret(\etc)}|E_n\right]\Pr(E_n)\\
	\nonumber&+\E\left[\textnormal{Regret(\etc)}|\neg E_n\right]\Pr(\neg E_n)\\
	\nonumber\leq&\E\left[\textnormal{Regret(\etc)}|E_n\right]+\E\left[\textnormal{Regret(\etc)}|\neg E_n\right]\Pr(\neg E_n)\\
	\nonumber\leq&\left(n\cdot d\Delta_{\max}+T\left(\langle \a_{I_{t_0}},\bmu\rangle-\langle \a_*,\bmu\rangle\right)\right)+T\Delta_{\max}\delta\\
	\label{eq:etcthm3}\leq&nd^2+TSR\sqrt{\frac{2\ln (2)d^2+4d\ln{\delta^{-1}}}{n}}+Td\delta\\
	\nonumber=&\widetilde{O}\left(d^{\frac{4}{3}}T^{\frac{2}{3}}\right).
	\end{align}
	Here inequality (\ref{eq:etcthm3}) follows from inequality (\ref{eq:etcthm2}), and we take $n=d^{-2/3}T^{2/3},$ and $\delta=1/(dT)$ in the last step. 
\end{proof}

\subsection{Proof of Lemma \ref{lemma:ttc_intuition}}
\label{sec:lemma:ttc_intuition}
\begin{proof}
	The discussion is conditioned on the event $E.$ throughout the proof On one hand, if (\ref{checking}) holds, then the detected path $\a_k$ cannot be optimal, \ie, given $$\langle\a_k,\hbmu_{m}\rangle-\langle\tilde{\a}_m,\hbmu_m\rangle\geq2\tilde{\Delta}_m,$$ we have
	\begin{align}
	\label{checking1}\langle\a_k,\bmu\rangle\geq&\langle\a_k,\hbmu_{m}\rangle-\tilde{\Delta}_m\\
	\label{checking2}>&\langle\tilde{\a}_m,\hbmu_m\rangle+\tilde{\Delta}_m\\
	\label{checking3}\geq&\langle\tilde{\a}_m,\bmu\rangle,
	\end{align} 
	where inequalities (\ref{checking1}) and (\ref{checking3}) hold because we have conditioned on $E,$ and inequality (\ref{checking2}) follows from re-arranging the terms in (\ref{checking}). This implies that routing a packet via path $\a_k$ incurs more delay than that of $\tilde{a}_m,$ and it thus cannot be optimal; On the other hand, this criterion (\ref{checking}) also promises that any path $\a_k$ with a gap $$\Delta_k>4\tilde{\Delta}_m$$ is detected after epoch $m,$ \ie,
	\begin{align}
	\label{checking4}\langle\a_k,\hbmu_{m}\rangle-\tilde{\Delta}_m\geq&\langle\a_k,\bmu\rangle-2\tilde{\Delta}_m\\
	\label{checking5}=&\langle\a_*,\bmu\rangle+\Delta_k-2\tilde{\Delta}_m\\
	\label{checking6}>&\langle\a_*,\bmu\rangle+2\tilde{\Delta}_m\\
	\label{checking7}\geq&\langle\a_*,\hbmu_m\rangle+\tilde{\Delta}_m\\
	\label{checking8}\geq&\langle\tilde{\a}_m,\hbmu_m\rangle+\tilde{\Delta}_m
	\end{align}
	where inequalities (\ref{checking4}) and (\ref{checking7}) hold because we have conditioned on $E,$ equality (\ref{checking5}) holds by definition of $\Delta_k,$  inequality (\ref{checking6}) follows from the assumption that $\Delta_k>4\tilde{\Delta}_m,$ and inequality (\ref{checking8}) follows from the optimality of $\tilde{\a}_m.$ This is equivalent to 
	\begin{align}
	\langle\a_k,\hbmu_{m}\rangle-\langle\tilde{\a}_m,\hbmu_m\rangle>&2\tilde{\Delta}_m,
	\end{align}
	and $\a_k$ is detected.
\end{proof}

\subsection{Proof of Theorem \ref{ttcthm}}
\label{sec:ttcthm}
\begin{proof}
	We begin by decomposing the regret as following
	\begin{align}
	\nonumber&\E\left[\textnormal{Regret}_T\left(\text{\ttc}\right)\right]\\
	\nonumber=&\E\left[\textnormal{Regret}_T\left(\text{\ttc}\right)| E\right]\Pr(E)\\
	\nonumber&\quad+\E\left[\textnormal{Regret}_T\left(\text{\ttc}\right)|\neg E\right]\Pr(\neg E)\\
	\nonumber\leq&\E\left[\textnormal{Regret}_T\left(\text{\ttc}\right)| E\right]\\
	\nonumber&\quad+\frac{T^2\delta}{d}\E\left[\textnormal{Regret}_T\left(\text{\ttc}\right)|\neg E\right],
	\end{align}
	and then distinguish the following two cases:
	
	\noindent\textbf{1. Analyzing $\E\left[\textnormal{Regret}_T\left(\text{\ttc}\right)|E\right]$}
	
	Under this case, all the sub-optimal arms should be eliminate when $$\tilde{\Delta}_m\leq\Delta_{\min}/4,$$ or $$m=(256dR^2\ln\delta^{-1}+128\ln(2)d^2R^2)/\Delta^2_{\min}.$$ Therefore,
	\begin{align}
	\nonumber&\E\left[\textnormal{Regret}_T\left(\text{\ttc}\right)|\neg E\right]\\
	\leq&\frac{256d^2R^2\Delta_{\max}\ln\delta^{-1}+128\ln(2)d^3R^2\Delta_{\max}}{{\Delta}_{\min}^2}.
	\end{align}
	
	\noindent\textbf{2. Analyzing $T^2\delta\E\left[\textnormal{Regret}_T\left(\text{\ttc}\right)|\neg E\right]$}
	
	We know that the regret of each round is at most $\Delta_{\max},$ and the total regret can be trivially upper bounded by $T\Delta_{\max}\leq Td.$ Therefore,
	\begin{align}
	\frac{T\delta}{d}\E\left[\textnormal{Regret}_T\left(\text{\ttc}\right)| E\right]\leq T^2\delta.
	\end{align}
	
	Combining the above two cases, we can sett $\delta$ to $T^{-2},$ and have
	\begin{align}
	\nonumber&\E\left[\textnormal{Regret}_T\left(\text{\ttc}\right)\right]\\
	\leq&\frac{512d^2R^2\Delta_{\max}\ln T+128\ln(2)d^3R^2\Delta_{\max}}{{\Delta}_{\min}^2}+1.
	\end{align} 
\end{proof}
\subsection{Proof of Theorem \ref{mttcthm}}
\label{sec:mttcthm}
\begin{proof}
	Conditioned on $E,$ The first half of the theorem follows directly from Theorem \ref{ttcthm}. Now suppose $\Delta_{\min}\leq d^{3/2}{T^{-1/4}},$ then the \mttc~switches to the alternative. The regret of in the first $n$ epochs can be upper bounded as 
	\begin{align}
	\Delta_{\max}d\overline{n}\leq d^2\overline{n}=\sqrt{T}R^2\left(32\ln (2)d+32\ln{T}\right).
	\end{align}
	The regret of running the alternative can thus be upper bounded as $C'd\sqrt{T\ln T}$ with some absolute constant $C'.$ Therefore, the worst case regret of the \mttc~can be upper bounded as
	\begin{align}
	\nonumber&\E\left[\textnormal{Regret}_T\left(\text{\ttc}\right)\right]\\
	\nonumber=&\E\left[\textnormal{Regret}_T\left(\text{\ttc}\right)| E\right]\Pr(E)\\
	\nonumber&\quad+\E\left[\textnormal{Regret}_T\left(\text{\ttc}\right)|\neg E\right]\Pr(\neg E)\\
	\nonumber\leq&\E\left[\textnormal{Regret}_T\left(\text{\ttc}\right)| E\right]\\
	\nonumber&\quad+\frac{T^2\delta}{d}\E\left[\textnormal{Regret}_T\left(\text{\ttc}\right)|E\right]\\
	\nonumber\leq&C'd\sqrt{T\ln T}+\frac{T^2\delta}{d}Td\\
	\nonumber\leq&C'd\sqrt{T\ln T}+1,
	\end{align}
	where we have again chosen $\delta=T^{-3}.$
\end{proof}
\subsection{Proof of Lemma \ref{general_nework:correctness}}
\label{sec:general_nework:correctness}
\begin{proof}
	We first note that in each iteration of while-loop in Algorithm \ref{general_network:alg}, either $u$ is increased by $1$ or the variable Flag is set to False. Therefore, after at most $d$ iterations, the algorithm terminates. We then see that the statement $\B$ has linearly independent columns holds trivially as $\C_u$ has linearly independent columns by virtue of our algorithm, and $\B$ is just a sub-matrix of $\C_u.$ 
	
	As an intermediate step, we show $\text{rank}(\B)=\text{rank}(\A).$ Since $\B$ is always a subset of $\A,$ the rank of $\B$ cannot exceed that of $\A.$ Now if $\text{rank}(\B)<\text{rank}(\A)=d_0,$ then there must exists an $\a\in\A,$ such that $\a$ is linearly independent of the columns of $\B.$ We declare that this is impossible once the algorithm terminate after $u$ iterations. Upon termination, the matrix $(\a,\C_u)$ has rank $d$ as $\C_u$ is full rank. By definition of $\a,$ the rank of $(\a,\B)$ is $\text{rank}(B)+1\leq d_0<d,$ and therefore, we must be able to start from $(\a,\B),$ and add columns from $\C_u(:,[u+1,d])$ to form a $d$-by-$d$ full rank matrix by basis extension theorem. This is equivalent to replace a column of $\C_u(:,[u+1,d])$ with $\a$ while keeping resulted matrix full rank, which is exactly step (2) of the procedure. This means the algorithm should not terminate, and it is a contradiction. Therefore, $\text{rank}(\B)=\text{rank}(\A).$
	
	Now for any $\a\in\A,$ if $\a$ cannot be expressed as linearly combination of columns of $\B,$ then adding $\a$ to $\B$ has rank $d_0+1.$ As $\left(\B,\a\right)$ is a sub-matrix of $\A,$ we have $$d_0+1=\text{rank}\left(\B,\a\right)\leq\text{rank}(\A)=d_0,$$ which is a contradiction.
\end{proof}

\subsection{Proof of Lemma \ref{ad0}}
\label{sec:ad0}
\begin{proof}
	We first show that the noise vector $\w_{t}=(\eta_1,\ldots,\eta_{t})$ is $R$-sub-Gaussian for all $t\in[T]$. It is easy to see that $\E\left[\w_t\right]=0,$ and for any $\bm{\alpha}\in\Re^{t},$ we have
	\begin{align}
	\nonumber&\E\left[\exp\left(\bm{\alpha}^{\top}\w_t\right)\right]\\
	\nonumber=&\E\left[\E\left[\exp\left(\sum_{s=1}^{t-1}\alpha_s\eta_s+\alpha_t\eta_t\right)\middle|\a_{I_1},\ldots,\a_{I_t},\eta_1,\ldots,\eta_{t-1}\right]\right]\\
	\nonumber\leq&\E\left[\E\left[\exp\left(\frac{\alpha_t^2R^2}{2}\right)\exp\left(\sum_{s=1}^{t-1}\alpha_s\eta_s\right)\middle|\a_{I_1},\ldots,\a_{I_t},\eta_1,\ldots,\eta_{t-1}\right]\right]\\
	\nonumber=&\exp\left(\frac{\alpha_t^2R^2}{2}\right)\E\left[\exp\left(\sum_{s=1}^{t-1}\alpha_s\eta_s\right)\right]\\
	\nonumber&\vdots\\
	=&\exp\left(\frac{\|\bm{\alpha}\|^2R^2}{2}\right).
	\end{align}
	Now from Theorem 2.2 of \cite{RH17}, we have the probability that the difference between $\hbmu_m$ and $\bmu$ under the $\V_m$ norm is not less than $R\sqrt{32\ln (6)d_0+32\ln{\delta^{-1}}}$ is at most $\delta,$ \ie, 
	\begin{align} \textstyle
		\nonumber\Pr\left(\|\hbmu_m-\bmu\|_{\V_m}\geq R\sqrt{32\ln (6)d_0+32\ln{\delta^{-1}}}\right)\leq\delta.
	\end{align}
	Equivalently, we have 
	\begin{align}
	\label{ad01}
	\left(\hbmu_m-\bmu\right)^{\top}\V_m\left(\hbmu-\bmu\right)\leq\gamma^2,
	\end{align}
	where $$\gamma=R\sqrt{32\ln (6)d_0+32\ln{\delta^{-1}}}.$$ By definition of $\B$ and $\V_m,$ we know that $\V_m=m\B\B^{\top}.$ Denoting $$\x=\B^{\top}\left(\hbmu_m-\bmu\right),$$ (\ref{ad01}) indicates with probability at least $1-\delta,$
	\begin{align}
	\|\x\|^2\leq\frac{\gamma^2}{m}.
	\end{align}
	As $\B$ is the $S$-approximate barycentric spanner of $\A,$ for any $\a\in\A,$ there exists some $\bnu_{\a}\in[-S,S]^{d_0},$ such that $\B\bnu_{\a}=\a.$ Therefore,
	\begin{align}
	\nonumber\forall\a\in\A\quad\langle\a,\hbmu_m-\bmu\rangle^2=&\left(\hbmu_m-\bmu\right)^{\top}\B\bnu_{\a}\bnu_{\a}^{\top}\B^{\top}\left(\hbmu_m-\bmu\right)\\
	\nonumber=&\x^{\top}\bnu_{\a}\bnu_{\a}^{\top}\x\\
	\nonumber=&\left(\x^{\top}\bnu_{\a}\right)^2\\
	\nonumber\leq&\|\x\|^2\|\bnu_{\a}\|^2\\
	\nonumber\leq&\frac{d_0S^2\gamma^2}{m}
	\end{align}
	holds with probability at least $1-\delta.$ Here the second last inequality follows from Cauchy-Schwarz inequality. This is equivalent to
	\begin{align*}
	\Pr\left(\left|\langle \a,\bmu\rangle-\langle \a,\hbmu\rangle\right|\geq SR\sqrt{\frac{32\ln (6)d_0^2+32d_0\ln{\delta^{-1}}}{m}}\right)\leq\delta.
	\end{align*}
\end{proof}
\section{Algorithm for Finding Barycentric Spanners}
\label{sec:barycentric_app}
In this section, we briefly state the algorithm from \cite{AK04} to illustrate how to construct a $S$-approximate barycentric spanner when the network is identifiable. The detail is shown in Algorithm \ref{alg:barycentric}. Please refer to Section \ref{sec:matrix_notation} for the matrix notations.  Here, $\B(:,-j)$ is the matrix $\B$ with the $j^{\text{th}}$ column removed. Each iteration of the for- and while-loop requires two quantities, \ie,
\begin{align}
&\argmax_{\a\in\A}\det(\a,\B(:,-j)),\\
&\argmax_{\a\in\A}-\det(\a,\B(:,-j)),
\end{align}
to compute $\argmax_{\a\in\A}\det|(\a,\B(:,-j))|,$ This can be done by two calls to the longest path algorithm for directed acyclic graphs. Specifically, the each coefficient of the linear function $\det(\a,\B(:,-j)$ is the determinant of a $(d-1)$-by-$(d-1)$ sub-matrix of $\B(:,−j)$ and can therefore be computed efficiently. Afterwards, we can set the each link's delay to the corresponding coefficient in $G,$ and run the longest path algorithm over this network to find $\a.$
 \begin{algorithm}[!ht]
	\caption{Basis Identification for Identifiable Networks \cite{AK04}}
	\label{alg:barycentric}
	\begin{algorithmic}[1]
		\State \textbf{Input:} A set of paths $\A.$
		\State \textbf{Initialization:} $\B\leftarrow I.$
		\For{$j=1,\ldots,d$}
		\State{$B(:,j)\leftarrow\argmax_{\a\in\A}\det|\left(\a,\B(:,-j)\right)|.$}
		\EndFor
		\While{$\exists \a\in\A,j\in[d]$ \st, $\det|\left(\a,\B(:,-j)\right)|>S\det|\B|$ }
		\State{$\B(:,j)\leftarrow\a.$}
		\EndWhile
		\Return{$\B.$}
	\end{algorithmic}
\end{algorithm}
\end{document}